\documentclass{article}

\usepackage[preprint]{nips_2018}

\usepackage[utf8]{inputenc} 
\usepackage[T1]{fontenc}    
\usepackage{hyperref}       
\hypersetup{%
	citecolor=black
}
\usepackage{url}            
\usepackage{booktabs}       
\usepackage{amsfonts}       
\usepackage{nicefrac}       
\usepackage{microtype}      
\usepackage{subdepth} 
\usepackage{amssymb} 

\usepackage{amsmath}
\usepackage{amsthm}
\usepackage{algpseudocode,algorithm,algorithmicx}
\usepackage{graphicx}
\graphicspath{ {images/} }
\usepackage[usenames, dvipsnames]{color}

\usepackage{mathtools}
\definecolor{dark-gray}{gray}{.35}
\definecolor{myorange}{RGB}{246, 164, 16}
\definecolor{mygreen}{RGB}{1, 100, 3}
\usepackage{footnote}
\makesavenoteenv{table}
\makesavenoteenv{tabular}
\usepackage{cleveref}[2012/02/15]
\crefformat{footnote}{#2\footnotemark[#1]#3}

\DeclareMathOperator*{\argsup}{arg\,sup}

\DeclareMathOperator*{\sigmoid}{\,sigmoid}

\theoremstyle{definition}
\newtheorem{definition}{Definition}[section]
\newtheorem{theorem}{Theorem}[section]
\newtheorem{corollary}{Corollary}[section]
\newtheorem{lemma}[theorem]{Lemma}

\title{On Relativistic $f$-Divergences}

\author{
  Alexia Jolicoeur-Martineau \\
  Lady Davis Institute\\
  MILA, Université de Montréal\\
  Montréal, Canada\\
  \texttt{alexia.jolicoeur-martineau@mail.mcgill.ca} \\
}

\begin{document}

\maketitle

\begin{abstract}
This paper provides a more rigorous look at Relativistic Generative Adversarial Networks (RGANs). We prove that the objective function of the discriminator is a statistical divergence for any concave function $f$ with minimal properties ($f(0)=0$, $f'(0) \neq 0$, $\sup_x f(x)>0$). We also devise a few variants of relativistic $f$-divergences. Wasserstein GAN was originally justified by the idea that the Wasserstein distance (WD) is most sensible because it is weak (i.e., it induces a weak topology). We show that the WD is weaker than $f$-divergences which are weaker than relativistic $f$-divergences. Given the good performance of RGANs, this suggests that WGAN does not performs well primarily because of the weak metric, but rather because of regularization and the use of a relativistic discriminator. We also take a closer look at estimators of relativistic $f$-divergences. We introduce the minimum-variance unbiased estimator (MVUE) for Relativistic paired GANs (RpGANs; originally called RGANs which could bring confusion) and show that it does not perform better. Furthermore, we show that the estimator of Relativistic average GANs (RaGANs) is only asymptotically unbiased, but that the finite-sample bias is small. Removing this bias does not improve performance.
\end{abstract}

\section{Introduction}

Generative adversarial networks (GANs) \citep{GAN} are a very popular approach to approximately generate data from a complex probability distribution using only samples of data (without any information on the true data distribution). Most notably, it has been very successful at generating photo-realistic images \citep{karras2017progressive} \citep{karras2018style}. It consists in a game between two neural networks, the generator $G$ and the discriminator $D$. The goal of $D$ is to classify real from fake (generated) data. The goal of $G$ is to generate fake data that appears to be real, thus "fooling" $D$ into thinking that fake data is actually real.

There are many GANs variants and most of them consist of changing the loss function of $D$. To name a few: Standard GAN (SGAN) \citep{GAN}, Least-Squares GAN (LSGAN) \citep{LSGAN}, Hinge-loss GAN (HingeGAN) \citep{miyato2018spectral}, Wasserstein GAN (WGAN) \citep{WGAN}. 

For most GAN variants, training $D$ is equivalent to estimating a divergence: SGAN estimates the Jensen–Shannon divergence (JSD), LSGAN estimates the Pearson $\chi^2$ divergence, HingeGAN estimates the Reverse-KL divergence, and WGAN estimates the Wasserstein distance. Even more generally, $f$-GANs \citep{F-GAN} estimate any $f$-divergence (which includes most of the popular divergences), while IPM-based GANs \citep{Fisher} estimate any Integral probability metric (IPM) \citep{muller1997integral}. Thus, intuitively, GANs can be thought as approximately minimizing a divergence (this is not technically correct; see \citet{jolicoeur2018beyonddivergence}).

Recently, \citet{RGAN} showed that IPM-based GANs possess a unique type of discriminator which they call a Relativistic Discriminator (RD). They explained that one can construct $f$-GANs while using a RD and that doing so improves the stability of the training and quality of generated data. They called this approach Relativistic GANs (RGANs). They proposed two variants: Relativistic paired GANs (RpGANs)\footnote{We added the word "paired" to better distinguish the variant with paired real/fake data (originally called RGANs) and the general approach called Relativistic GANs (RGANs).} and Relativistic Average GANs (RaGANs). 

\citet{RGAN} provided mathematical and intuitive arguments as to why having a Relativistic Discriminator (RD) may be helpful. However, they did not show that the loss functions are mathematically sensible as they did not show that these form statistical divergences. Furthermore, the estimators that they used were not the minimum-variance unbiased estimators (MVUE).

The contributions of this paper are the following:
\begin{enumerate}
	\item We prove that the objective functions of the discriminator in RGANs are divergences (relativistic $f$-divergences).
	\item We devise a few variants of Relativistic $f$-divergences.
	\item We show that the Wasserstein Distance is weaker than $f$-divergences which are weaker than relativistic $f$-divergences.
	\item We present the minimum-variance unbiased estimator (MVUE) of RpGANs and show that using it hinders the performance of the generator.
	\item We show that RaGANs are only asymptotically unbiased, but that the finite-sample bias is small. Removing this bias does not improve the performance of the generator.
\end{enumerate}

\section{Background}

For the rest of the paper, we focus on the critic $C(x)$ instead of the discriminator $D(x)$. The critic is the discriminator before applying the activation function ($D(x)=a(C(x))$, where $a$ is an activation function and $C(x) \in \mathbb{R}$). Intuitively, the critic can be thought as describing how realistic $x$ is. In the case of SGAN and HingeGAN, a large $C(x)$ means that $x$ is realistic, while a small $C(x)$ means that $x$ is not realistic. We use this notation because Relativistic GANs are defined in terms of the critic rather than the discriminator.

\subsection{Generative adversarial networks}

GANs can be defined very generally in the following way:
\begin{equation}
\sup_{C: \mathcal{X} \to \mathbb{R}} \mathbb{E}_{x \sim \mathbb{P}}\left[ f_1(C(x)) \right] + \mathbb{E}_{y \sim \mathbb{Q}} \left[ f_2(C(y)) \right],
\end{equation}
\vspace*{-4pt}
\begin{equation}
\sup_{G: Z \to \mathcal{X}} \mathbb{E}_{x \sim \mathbb{P}}\left[ g_1(C(x)) \right] + \mathbb{E}_{z \sim \mathbb{Z}} \left[ g_2(C(G(z))) \right],
\end{equation}
where $f_1$, $f_2$, $g_1$, $g_2:\mathbb{R} \to \mathbb{R}$, $\mathbb{P}$ is the distribution of real data with support $\mathcal{X}$,  $\mathbb{Z}$ is the latent distribution (generally a multivariate normal distribution), $C(x)$ is the critic evaluated at $x$, $G(z)$ is the generator evaluated at $z$, and $G(z) \sim \mathbb{Q}$, where $\mathbb{Q}$ is the distribution of fake data. See \citet{BigGAN} for details on how different choices of $\mathbb{Z}$ performs. The critic and the generator are generally trained with stochastic gradient descent (SGD) in alternating steps.

Most GANs can be separated in two classes: non-saturating and saturating loss functions. GANs with the saturating loss are such that $g_1$=$-f_1$ and $g_2$=$-f_2$, while GANs with the non-saturating loss are such that $g_1$=$f_2$ and $g_2$=$f_1$. In this paper, we will assume that the non-saturating loss is always used as it generally works best in practice \citep{GAN} \citep{F-GAN}. Note that $g_1$ is also generally not included as its gradient with respect to $G$ is zero.

Although not always the case, the most popular GAN loss functions (SGAN, LSGAN with labels -1/1, HingeGAN, WGAN) are symmetric (i.e., $f_2(x) = f_1(-x)$). For simplicity, in this paper, we restrict ourselves to symmetric loss functions.

Non-saturating Symmetric GANs (SyGANs) can be represented more simply as:
\begin{equation}
\sup_{C: \mathcal{X} \to \mathbb{R}} \mathbb{E}_{x \sim \mathbb{P}}\left[ f(C(x)) \right] + \mathbb{E}_{y \sim \mathbb{Q}} \left[ f(-C(y)) \right],
\end{equation}
\vspace*{-4pt}
\begin{equation}
\sup_{G: Z \to \mathcal{X}} \mathbb{E}_{z \sim \mathbb{Z}} \left[ f(C(G(z))) \right],
\end{equation}
for some function $f:\mathbb{R}\to\mathbb{R}$. For easier optimization, we generally want $f$ to be concave with respect to the critic. This is the case in symmetric $f$-GANs since $f(x) = f_2(x)=-f^{*}(a(x))$, for some convex function $f^{*}$ and non-decreasing function $a(x)$, is concave.

In this paper, we restrict our relativistic divergences to symmetric cases with concave $f$. Although this may be somewhat constraining, not making these assumptions would be very problematic for GANs. By not assuming concavity, we could have an objective function that diverges to infinity (and thus an infinite divergence). This is particularly problematic for GANs because early in training, we expect $\mathbb{P}$ and $\mathbb{Q}$ to be perfectly separated (because of fully disjoint supports). This would cause the objective function to explode towards infinity and thereby causing severe instabilities. The Kullback–Leibler (KL) divergence is a good example of such a problematic divergence for GANs. If a single sample from the support of $\mathbb{Q}$ is not part of the support of $\mathbb{P}$, the divergence will be $\infty$. Also, note that the dual form of the KL divergence cannot be represented as a SyGAN with equation (3) since $f_1(x)=x$ and $f_2(x)=-e^{x-1}$ are not symmetric \citep{F-GAN}.

\subsection{Integral probability metrics}

Rather than using a concave function $f$ to ensure a maximum on the objective function, IPM-based GANs instead force the critic to respect some constraint so that it does not grow too quickly. IPM-based GANs are defined in the following way:
\begin{equation}
\underset{ \substack{C:\mathcal{X} \to \mathbb{R} \vphantom{p}  \\ C \in \mathcal{F}}}{\sup} \mathbb{E}_{x \sim \mathbb{P}}\left[ C(x) \right] - \mathbb{E}_{y \sim \mathbb{Q}} \left[C(y) \right],
\end{equation}
\vspace*{-4pt}
\begin{equation}
\sup_{G: Z \to \mathcal{X}} \mathbb{E}_{z \sim \mathbb{Z}} \left[ C(G(z)) \right],
\end{equation}
where $\mathcal{F}$ is a class of IPM. See \citet{mroueh2017sobolev} for an extensive review of the choices of $\mathcal{F}$.

\subsection{Relativistic GANs}

Rather than training the critic on real and fake data separately, this approach tries to maximize the critic's difference (CD), but not too much. In Relativistic paired GANs (RpGANs), the CD is defined as $C(x)-C(y)$, while in Relativistic average GANs (RaGANs), the CD is defined as $C(x)-\underset{y \sim \mathbb{Q}}{\mathbb{E}}C(y)$ (or vice-versa). The CD can be understood as how much more realistic real data is from fake data. The optimal size of the CD is determined by the choice of $f$. With a least-square loss, the CD must be exactly equal to 1. On the other hand,  with a log-sigmoid loss, the CD is grown to around 2 or 3 (after-which the gradient of $f$ vanishes to zero). This will be explained in more details in the next section. Again, we focus only on cases with symmetry (as done with SyGANs).

Relativistic paired GANs (RpGANs) are defined in the following way:
\begin{equation}
\sup\limits_{C:\mathcal{X} \to \mathbb{R}} \hspace{1pt}
\underset{ \substack{x \sim \mathbb{P}  \\ y \sim \mathbb{Q}}}{\mathbb{E}\vphantom{p}} \left[ f \left( C(x) - C(y) \right) \right],
\end{equation}
\vspace*{-4pt}
\begin{equation}
\sup\limits_{G: Z \to \mathcal{X}} \hspace{1pt}
\underset{ \substack{x \sim \mathbb{P}  \\ z \sim \mathbb{Z}}}{\mathbb{E}\vphantom{p}} \left[ f \left( C(G(z)) - C(x) \right) \right].
\end{equation}

Relativistic average GANs (RaGANs) are defined in the follow way:
\begin{equation}
\sup\limits_{\scriptstyle C:\mathcal{X} \to \mathbb{R}} \hspace{1pt}
\underset{\scriptstyle x \sim \mathbb{P}}{\mathbb{E}\vphantom{p}} \left[ f \left( C(x) - \underset{y \sim \mathbb{Q}}{\mathbb{E}}C(y) \right) \right] +
\underset{\scriptstyle y \sim \mathbb{Q}}{\mathbb{E}\vphantom{p}} \left[ f \left( \underset{x \sim \mathbb{P\vphantom{Q}}}{\mathbb{E}}C(x) - C(y) \right) \right],
\end{equation}
\vspace*{-4pt}
\begin{equation}
\sup\limits_{\scriptstyle G: Z \to \mathcal{X}} \hspace{1pt}
\underset{\scriptstyle z \sim \mathbb{Z}}{\mathbb{E}\vphantom{p}} \left[ f \left( C(G(z)) - \underset{x \sim \mathbb{P}}{\mathbb{E}}C(x) \right) \right] +
\underset{\scriptstyle x \sim \mathbb{P}}{\mathbb{E}\vphantom{p}} \left[ f \left( \underset{z \sim \mathbb{P\vphantom{Q}}_z}{\mathbb{E}}C(G(z)) - C(x) \right) \right].
\end{equation}

\section{Relativistic Divergences}

We define statistical divergences in the following way:

\theoremstyle{definition}
\begin{definition}\label{1.0}
	Let $\mathbb{P}$ and $\mathbb{Q}$ be probability distributions and $S$ be the set of all probability distributions with common support. A function $D:(S,S) \to \mathbb{R}_{>0} $ is a divergence if it respects the following two conditions:
	\begin{align*}
	&D(\mathbb{P}, \mathbb{Q}) \ge 0 \\ 
	&D(\mathbb{P}, \mathbb{Q}) = 0 \iff \mathbb{P} = \mathbb{Q}.
	\end{align*}
\end{definition}
In other words, divergences are distances between probability distributions. The distribution of real data ($\mathbb{P}$) is fixed and our goal is to modify the distribution of fake data ($\mathbb{Q}$) so that the divergence decreases over training time.

\subsection{Main theorem}

As discussed in the introduction, in most GANs, the objective function of the critic at optimum is a divergence. We show that the objective function of the critic in RpGANs, RaGANs, and other variants also estimate a divergence. The theorem is as follows:
\begin{theorem}
	Let $f:\mathbb{R} \to \mathbb{R}$ be a concave function such that $f(0)=0$, $f$ is differentiable at 0, $f'(0)\neq0$, $\sup_x f(x) = M > 0$, and $\argsup_x f(x) > 0$. Let $\mathbb{P}$ and $\mathbb{Q}$ be probability distributions with support $\mathcal{X}$. Let $\mathbb{M} = \frac{1}{2}\mathbb{P} + \frac{1}{2}\mathbb{Q}$. Then, we have that 
	\begin{align*}
	\mathrm{D}^{Rp}_f(\mathbb{P}, \mathbb{Q}) &= \sup\limits_{C:\mathcal{X} \to \mathbb{R}} \hspace{1pt}
	2\underset{ \substack{x \sim \mathbb{P}  \\ y \sim \mathbb{Q}}}{\mathbb{E}\vphantom{p}} \left[ f \left( C(x) - C(y) \right) \right]	\\
	\mathrm{D}^{Ra}_{f}(\mathbb{P}, \mathbb{Q}) &= \sup\limits_{\scriptstyle C:\mathcal{X} \to \mathbb{R}} \hspace{1pt}
	\underset{\scriptstyle x \sim \mathbb{P}}{\mathbb{E}\vphantom{p}} \left[ f \left( C(x) - \underset{y \sim \mathbb{Q}}{\mathbb{E}}C(y) \right) \right] +
	\underset{\scriptstyle y \sim \mathbb{Q}}{\mathbb{E}\vphantom{p}} \left[ f \left( \underset{x \sim \mathbb{P\vphantom{Q}}}{\mathbb{E}}C(x) - C(y) \right) \right] \\
	\mathrm{D}^{Ralf}_{f}(\mathbb{P}, \mathbb{Q}) &= \sup \limits_{\scriptstyle C:\mathcal{X} \to \mathbb{R}} \hspace{1pt}
	2\underset{\scriptstyle x \sim \mathbb{P}}{\mathbb{E}\vphantom{p}} \left[ f \left( C(x) - \underset{y \sim \mathbb{Q}}{\mathbb{E}}C(y) \right) \right] \\
	\mathrm{D}^{Rc}_{f}(\mathbb{P}, \mathbb{Q}) &= \sup\limits_{\scriptstyle C:\mathcal{X} \to \mathbb{R}} \hspace{1pt}
	\underset{\scriptstyle x \sim \mathbb{P}}{\mathbb{E}\vphantom{p}} \left[ f \left( C(x) - \underset{m \sim \mathbb{M}}{\mathbb{E}}C(m) \right) \right] +
	\underset{\scriptstyle y \sim \mathbb{Q}}{\mathbb{E}\vphantom{p}} \left[ f \left( \underset{m \sim \mathbb{M}}{\mathbb{E}}C(m) - C(y) \right) \right]
	\end{align*}
	are divergences.
\end{theorem}

We ask that the supremum of $f(x)$ is reached at some positive $x$ (or at $\infty$). This is purely to ensure that a larger CD can be interpreted as leading to a larger divergence (rather than the opposite). This does not reduce the generality of Theorem 3.1. If $f(x)$ is maximized at $x < 0$, we have that $g(x) = f(-x)$ is maximized at $x > 0$ and one can simply use $g$ instead of $f$.

We require that $f$ is differentiable at zero and its derivative to be non-zero. This assumption may not be necessary, but it is needed for one of our main lemma which we use to prove that these objective functions are divergences.

A one-page sketch of the proof is available in Appendix A; the full proof is found in Appendix B. 

Note that $\mathrm{D}^{Rp}_f(\mathbb{P}, \mathbb{Q})$ corresponds to RpGANs, $\mathrm{D}^{Ra}_{f}(\mathbb{P}, \mathbb{Q})$ corresponds to RaGANs,  $\mathrm{D}^{Ralf}_{f}(\mathbb{P}, \mathbb{Q})$ corresponds to a simplified one-way version of RaGANs (RalfGANs), and $\mathrm{D}^{Rc}_{f}(\mathbb{P}, \mathbb{Q})$ corresponds to a new type of RGAN called Relativistic centered GANs (RcGANs).  RalfGANs are not particularly interesting as they simply represent a simpler version of RaGANs. On the other hand, RcGANs are interesting as they center the critic scores using the mean of the whole mini-batch (rather than the mean of only real or only fake mini-batch samples). This divergence also has similarities to the Jensen–Shannon divergence (JSD) since the JSD adds the divergence between $\mathbb{P}$ and $\mathbb{M}$ to the divergence between $\mathbb{Q}$ and $\mathbb{M}$.

A logical extension to RcGANs would be to standardize the critic scores; however, this would not lead to a divergence given that we could not control the size of the elements inside $f$. To make it a divergence, we would need a learn-able scaling weight (as in batch norm \citep{BatchNorm}), but this would counter the effect of the standardization. Thus, standardizing and scaling would just correspond to an equivalent re-parametrization of $\mathrm{D}^{Rc}_{f}$.

\subsection{Subtypes of divergences}

\begin{figure}[!ht]
	\centering
	\includegraphics[scale=0.7]{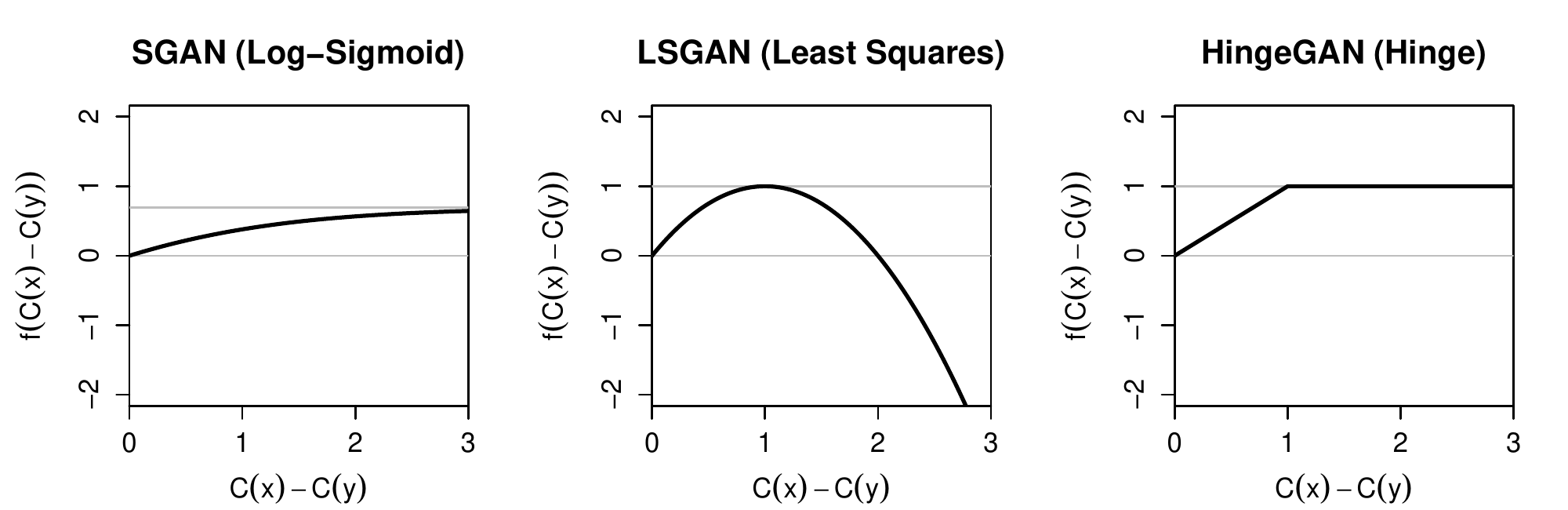}
	\caption{Plot of $f$ with respect to the critic's difference (CD) using three appropriate choices of $f$ for relativistic divergences. The bottom gray line represents $f(0)=0$; the divergence is zero if all CDs are zero. The above gray line represents the maximum of $f$; the divergence is maximized if all CDs leads to that maximum.}
\end{figure}

Figure 1 shows three examples of concave $f$ with the necessary properties to be used in relativistic divergences; they are the concave functions used in SGAN, LSGAN (with labels 1/-1), and HingeGAN. Their respective mathematical functions are 
\begin{align}
f_{S}(z) &= \log(\sigmoid(z))+\log(2), \\
f_{LS}(z) &= -(z-1)^2 + 1, \\ 
f_{Hinge}(z) &= - \max(0,1-z)+1 .
\end{align}

 Interestingly, we see that they form three different types of functions. Firstly, we have functions that grow exponentially less as $x$ increases and thus reach their supremum at $\infty$. Secondly, we have functions that grow to a maximum and then forever decrease (thus penalizing large CDs). Thirdly, we have functions that grow to a maximum and then never change. SGAN is of the first type, LSGAN is of the second, and HingeGAN is of the third type.

This shows that for all three types, we have that the CD is only encouraged to grow until a certain point. With the first type, we never truly force the CD to stop growing, but the gradients vanish to zero. Thus, SGD effectively prevents the CDs from growing above a certain level (sigmoid saturates at around 2 or 3).

It is useful to keep in mind that Figure 1 also represents the concave functions used for SyGANs, in which case $f$ applies to real and fake data separately ($f(x)$ and $f(-y)$). 

\subsection{Weakness of the divergence}

The paper by \citet{WGAN} on using the Wasserstein distance (and other IPMs) for GANs has been extremely influential. In the WGAN paper, the authors suggest that the Wasserstein distance is more appropriate than $f$-divergences for training a critic since it induces the weakest topology possible. Rather than giving a formal definition in terms of topologies, we show a simpler definition (as also done by \citet{WGAN}):

\theoremstyle{definition}
\begin{definition}\label{1.0}
	Let $\mathbb{P}$ be a probability distribution with support $\mathcal{X}$, $\left( \mathbb{P}_n\right)_{n \in \mathbb{N}}$ be a sequence of distributions converging to $\mathbb{P}$, and $D_1$ and $D_2$ be statistical divergences (per definition 3.1).
	
	We say that $D_1$ is weaker than $D_2$ if we have that:
	$$D_2(\mathbb{P}_n, \mathbb{P})\to 0 \implies D_1(\mathbb{P}_n, \mathbb{P})\to 0,$$
	but the converse is not true.
	
	We say that $D_1$ is a weakest distance if we have that:
	$$D_1(\mathbb{P}_n, \mathbb{P})\to 0 \iff \mathbb{P}_n \overset{D}{\to} \mathbb{P},$$
	where $\overset{D}{\to}$ represents convergence in distribution.
\end{definition}

Thus, intuitively, a weaker divergence can be thought as converging more easily. \citet{WGAN} showed that the Wasserstein distance is a weakest divergence and that it is weaker than common $f$-divergences (as used in $f$-GANs and standard GANs). They also showed that the Wasserstein distance is continuous with respect to its parameters and they attributed this property to the weakness of the divergence.

Considering this argument, one would except that RaGANs would be weaker than RpGANs which would be weaker than RGANs since  this is the order of their relative performance and stability. Instead, we found the opposite relationship:
\begin{theorem}
	Let $\mathbb{P}$ be a probability distribution with support $S$, $\left( \mathbb{P}_n\right)_{n \in \mathbb{N}}$ be a sequence of distributions converging to $\mathbb{P}$, $f:\mathbb{R} \to \mathbb{R}$ be a concave function such that $f(0)=0$, $f$ is differentiable at 0, $f'(0)\neq0$, $\sup_x f(x) = M > 0$, and $\argsup_x f(x) > 0$. Then, we have that 
	\begin{align*}
	&\mathrm{D}^{W}_f(\mathbb{P}, \mathbb{Q}) \text{ is weakest,} \\
	&\mathrm{D}^{W}_f(\mathbb{P}, \mathbb{Q}) \text{ is weaker than } \mathrm{D}^{Sy}_f(\mathbb{P}, \mathbb{Q}), \\ 
	&\mathrm{D}^{Sy}_f(\mathbb{P}, \mathbb{Q}) \text{ is weaker than } \mathrm{D}^{Rp}_f(\mathbb{P}, \mathbb{Q}), \\
	&\mathrm{D}^{Rp}_f(\mathbb{P}, \mathbb{Q}) \text{ is weaker than } \mathrm{D}^{Ra}_f(\mathbb{P}, \mathbb{Q}),
	\end{align*}
	were $\mathrm{D}^{W}$ is the Wasserstein distance and $\mathrm{D}^{Sy}$ is the distance in Symmetric GANs (see equation 3).
\end{theorem}
The proof is in Appendix C.

Given the good performance of RaGANs, this suggests that the argument made by \citet{WGAN} is insufficient. It only focuses on a perfect sequence of converging distributions, but the generator training does not guarantee a converging sequence of fake data distributions. It ignores the complex dynamics and intricacies of the generator training, which are still not well understood. Furthermore, it assumes an optimal critic which is unobtainable in practice. In practice, trying to obtain a semi-optimal critic requires many iterations and thus a significant amount of additional computational resources.

As previously suggested \citep{RGAN}, what makes the Wasserstein distance a good choice of divergence are likely 1) the constraint of the critic (a Lipschitz critic) and 2) the use of a relativistic discriminator, rather than the weakness of the divergence.

\section{Estimators}

\subsection{RpGANs}

To estimate RpGANs, \citet{jolicoeur2018beyonddivergence} used the following estimator\footnote{Note that they actually used $\frac{1}{k}$ instead of $\frac{2}{k}$ because of how they defined the divergence.}:
 $$\mathrm{\widehat{D}^{Rp}}_f(\mathbb{P},\mathbb{Q}) = \sup\limits_{\scriptstyle C:\mathcal{X} \to \mathbb{R}} \hspace{1pt}
\frac{2}{k} \sum_{i=1}^{k} \left[ f(C(x_i)-C(y_i)) \right],$$
where $x_1, \ldots , x_k$ and $y_1, \ldots , y_k$ are samples from $\mathbb{P}$ and $\mathbb{Q}$ respectively.

Although this is an unbiased estimator of $\mathrm{D}^{Rp}_{f}(\mathbb{P}, \mathbb{Q})$, it is not the estimator with the minimal variance for a given mini-batch. Using the two-sample version \citep{lehmann1951consistency} of the U-statistic theorem \citep{hoeffding1992class} and given that the loss function is symmetric with respect to its arguments, one can show the following:
\begin{corollary}
	Let $\mathbb{P}$ and $\mathbb{Q}$ be probability distributions with support $\mathcal{X}$. Let $x_1, \ldots , x_k$ and $y_1, \ldots , y_k$ be i.i.d. samples from $\mathbb{P}$ and $\mathbb{Q}$ respectively. Then, we have that 
	\begin{align*}
	\mathrm{\widehat{D}^{Rp*}}_f(\mathbb{P},\mathbb{Q}) &= \sup\limits_{\scriptstyle C:\mathcal{X} \to \mathbb{R}} \hspace{1pt}
	\frac{2}{k^2} \sum_{i=1}^{k} \sum_{j=1}^{k} \left[ f(C(x_i)-C(y_j)) \right]
	\end{align*}
	is the minimum-variance unbiased estimator (MVUE) of $\mathrm{D}^{Rp}_{f}(\mathbb{P}, \mathbb{Q})$.
\end{corollary}
Although it is the MVUE, this estimator requires $O(k^2)$ operations instead of $O(k)$. In the experiments, we will show that using this estimator does not lead to better results. Given the quadratic scaling and lack of performance gain, it may not be worth using.

\subsection{RaGANs and RalfGANs}

The divergences of RaGANs and RalfGANs assume that one knows the true expectation of the critic of real and fake data. However, in practice, we can only estimate the expectation. Although never explicitly mentioned, \cite{RGAN} simply replaced all expectations by the mini-batch mean: $$\mathbb{E}\left[ C(x) \right] \approx \frac{1}{k}\sum_{i=1}^{k} C(x_i), $$ where $k$ is the size of the mini-batch. 

Given the non-linear function applied after calculating the CD, the divergences of RaGANs are biased with finite batch size $k$. This means that RaGANs are only asymptotically unbiased. How large $k$ must be for the bias to become negligible is unclear. 

We attempted to find a close form for the bias with $f_S$, $f_{LS}$, and $f_{Hinge}$ (equations 11, 12, 13 and Figure 1). We were only able to find a closed form for the bias with $f_{LS}$. The bias with $f_{LS}$ has a simple form and can thus be removed, as seen below:
\begin{corollary}
	Let $\mathbb{P}$ and $\mathbb{Q}$ be probability distributions with support $\mathcal{X}$. Then, we have that 
	\begin{align*}
	&\sup\limits_{\scriptstyle C:\mathcal{X} \to \mathbb{R}} \hspace{1pt}
	\frac{1}{k} \left(\hat{\sigma}_{C(x)} + \hat{\sigma}_{C(y)} - \sum_{i=1}^{k} \left[ \left( C(x_i) -  \hat{\mu}_{C(y)} - 1  \right)^2 \right] -
	\sum_{j=1}^{k} \left[ \left( \hat{\mu}_{C(x)} - C(y_j) - 1  \right)^2 \right] \right)+ 2, \\ 
	&\sup \limits_{\scriptstyle C:\mathcal{X} \to \mathbb{R}} \hspace{1pt}
	\frac{2}{k} \left( \hat{\sigma}_{C(y)} - \sum_{i=1}^{k} \left[ \left( C(x_i) -  \hat{\mu}_{C(y)} - 1 \right)^2 \right] \right) + 1, \\
	&\sup\limits_{\scriptstyle C:\mathcal{X} \to \mathbb{R}} \hspace{1pt}
	-\frac{1}{k} \left(\frac{1}{2}\hat{\sigma}_{C(x)} + \frac{1}{2}\hat{\sigma}_{C(y)} + \sum_{i=1}^{k} \left[ \left( C(x_i) -  \hat{\mu}_{C} - 1  \right)^2 \right] +
	\sum_{j=1}^{k} \left[ \left( \hat{\mu}_{C} - C(y_j) - 1  \right)^2 \right] \right) + 2
	\end{align*}
	are unbiased estimator of $\mathrm{D}^{Ra}_{f_{LS}}(\mathbb{P}, \mathbb{Q})$, $\mathrm{D}^{Ralf}_{f_{LS}}(\mathbb{P}, \mathbb{Q})$, and $\mathrm{D}^{Rc}_{f_{LS}}(\mathbb{P}, \mathbb{Q})$ respectively,\\
	where $\hat{\mu}_{C(x)} = \frac{1}{k} \sum_{i=1}^{k} C(x_i)$, $\hat{\mu}_{C(y)} = \frac{1}{k} \sum_{i=1}^{k} C(y_i)$,  $\hat{\mu}_{C} = \frac{1}{k} \sum_{i=1}^{k} \left(\frac{C(x_i)+C(y_i)}{2}\right)$, $\hat{\sigma}_{C(x)} = \frac{1}{(k-1)}\sum_{i=1}^{k} \left( C(x_i) - \hat{\mu}_{C(x)}\right)^2$. and $\hat{\sigma}_{C(y)} = \frac{1}{(k-1)}\sum_{i=1}^{k} \left( C(y_i) - \hat{\mu}_{C(y)}\right)^2$.
\end{corollary}
See Appendix B for the proof. This means that we can estimate the loss function in RaLSGAN, RalfLSGAN, and RcLSGAN without bias. In the experiments, we will show that the bias is negligible with the usual choices of $f$ (equations 11, 12, 13) and batch size (32 or higher).

\section{Experiments}

All experiments were done with the spectral GAN architecture for 32x32 images (See \cite{miyato2018spectral}) in Pytorch \citep{pytorch}. We used the standard hyperparameters: learning rate (lr) = .0002, batch size (k) = 32, and the ADAM optimizer \citep{Adam} with parameters $(\alpha_1, \alpha_2)= (.50, .999)$. We trained the models for 100k iterations with one critic update per generator update. For the datasets, we used CIFAR-10 \citep{krizhevsky2009learning}, CelebA \citep{liu2015faceattributes} and CAT \citep{cat}. All models were trained using the same seed. To evaluate the quality of generated outputs, we used the Fréchet Inception Distance (FID) \citep{heusel2017gans}. For a review of the different evaluation metrics for GANs, please see \citet{borji2018pros}.

\subsection{Bias}

We approximated the bias of RaGANs and RcGANs by estimating the real/fake critic mean from $320$ samples rather than the $32$ mini-batch samples. For $f_{LS}$, we were able to calculate the true value of the bias (in expectation, see Corollary 4.2). Results on CIFAR-10 are shown in Figure 2. 

\begin{figure}[!ht]
	\centering
	\includegraphics[scale=0.5]{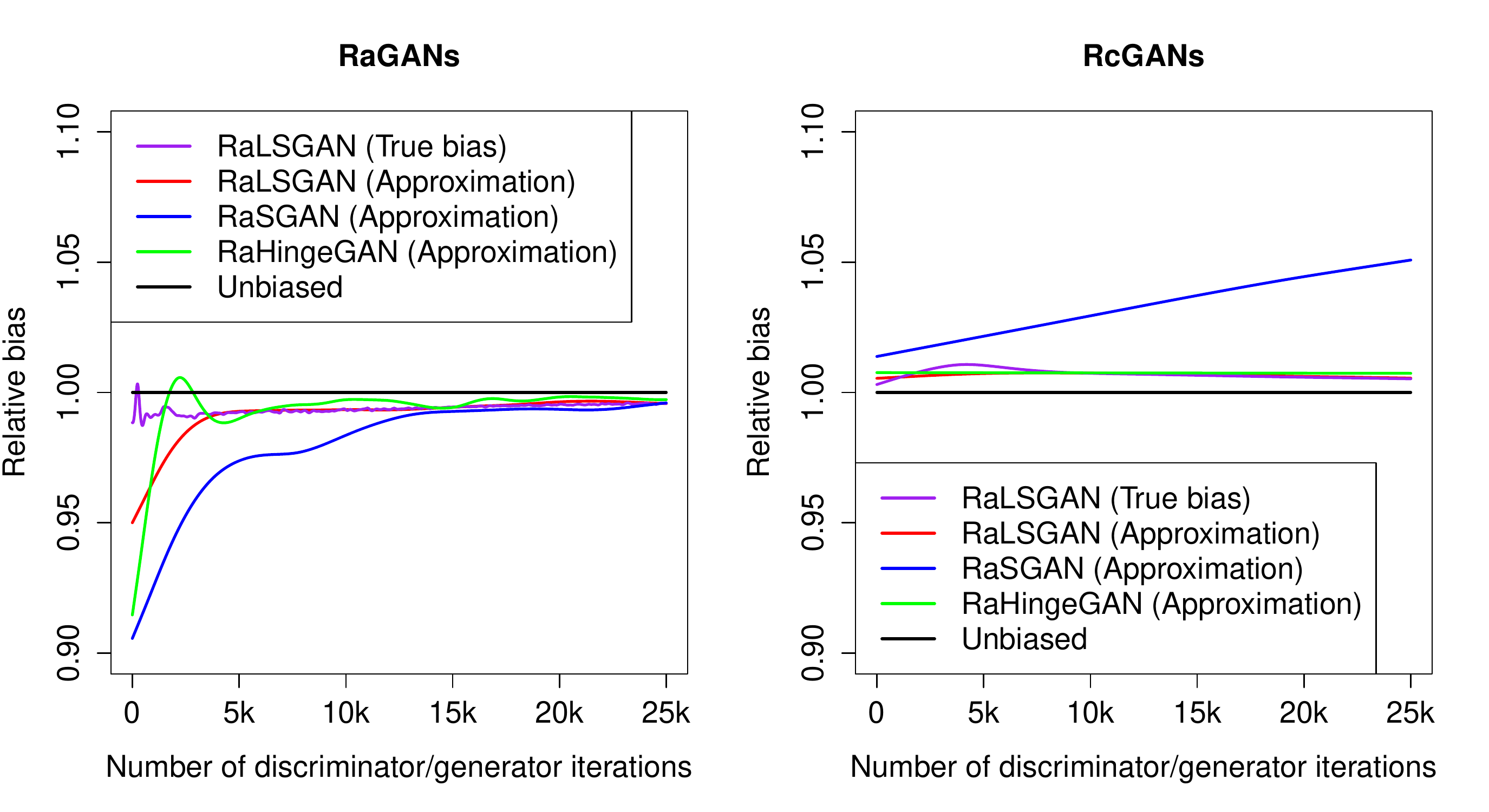}
	\caption{Plots of the relative bias (i.e., the biased estimate divided by the unbiased estimate) of relativistic average and centered $f$-divergences estimators over training time on CIFAR-10 with a mini-batch size of 32. Approximations of the bias were made using 320 independent samples.}
\end{figure}

For RAGANs, the approximation of the relative bias with $f_{LS}$ was correct from 4k iterations and onwards. For all choices of $f$, we observed the same pattern of low approximated relative bias that stabilized to a larger number after a certain number of iterations. We suspect that this may be due to the important instabilities of the first iterations when the discriminator is not optimal. At 15k iterations, all biases were stabilized. We calculated the average of the bias with different $f$ starting at 15k iterations: .995 for the true relative bias with $f_{LS}$, .996 for the approximated relative bias with $f_{LS}$, .994 for the approximated relative bias with $f_{S}$, and .997 for the approximated relative bias with $f_{Hinge}$.

For RcGANs, the approximation of the bias with $f_{LS}$ was correct from the very beginning of training. All biases were relatively stable over time with the exception of $f_{S}$ which increased linearly over time (up to around 1.05). We calculated the average of the bias with different $f$: 1.007 for the true relative bias with $f_{LS}$, 1.007 for the approximated relative bias with $f_{LS}$, 1.03 for the approximated relative bias with $f_{S}$, and 1.007 for the approximated relative bias with $f_{Hinge}$.

Overall, this shows that the bias in the estimators of RaGANs and RcGANs tends to be small. Furthermore, with the exception of $f_S$, the bias is relatively stable over time. Thus, accounting for the bias, may not be necessary.

\subsection{Divergences}

To test the new relativistic divergences proposed (and verify whether removing the bias in RaGANs is useful), we ran experiments on CIFAR-10 using $f_{LS}$, on LSUN bedrooms using $f_{Hinge}$, and on CAT using $f_{Hinge}$ (these choices of $f$ were arbitrary). Results are shown in Table 1.

\begin{table}[!ht]
	\caption{Minimum (and standard deviation) of the FID calculated at 10k, 20k, ... , 100k iterations using different loss functions (see equations 11, 12, 13) and datasets.}
	\label{CIFAR10}
	\centering
	\begin{tabular}{llll}
		\toprule
		& CIFAR-10 & CelebA & CAT \\
		Loss & $f_{LS}$ & $f_{Hinge}$ & $f_{S}$ \\
		\cmidrule(){1-4}
		GAN & 31.1 (8.2) & 15.3 (51.8) & 15.2 (11.1) \\
		RpGAN & 31.5 (7.6) & 16.7 (4) & 12.9 (2.3) \\
		RpGAN (MVUE) & 30.2 (11.7) & 21.9 (3.2) & 18.2 (2.9) \\
		RaGAN & {\fontseries{b}\selectfont 29.2} ({\fontseries{b}\selectfont 7.4}) & {\fontseries{b}\selectfont 15.9} (4.5)  & {\fontseries{b}\selectfont 12.3} ({\fontseries{b}\selectfont 1.2}) \\
		RaGAN (unbiased) & 30.3 (12.9) & - & -  \\
		RcGAN & 31.7 (8) & 18.1 ({\fontseries{b}\selectfont 2.9})  & 16.5 (7.1) \\
		RcGAN (unbiased) & 32.3 (8.7) & - & -  \\
		\bottomrule
	\end{tabular}
\end{table}

Using the MVUE for RpGAN resulted in the generator having a worse performance on CIFAR-10 with  $f_{LS}$ ($\beta=.37$, $p=.72$), CelebA with $f_{Hinge}$ ($\beta=2.08$, $p=.07$), and CAT with $f_{S}$ ($\beta=4.02$, $p=.003$). Similarly, using the unbiased estimator made the generator perform sightly worse for RaLSGAN ($\beta=2.37$, $p=.04$) and RcLSGAN ($\beta=1.33$, $p=.05$). These results are surprising as they suggest that using noisy or slightly biased estimators may be beneficial.

\section{Conclusion}

Most importantly, we proved that the objective function of the critic in RGANs is a divergence. 

In addition, we showed that $f$-divergences are weaker than relativistic $f$-divergences. Thus, the weakness of the topology induced by a divergence alone cannot explain why WGAN performs well.

Finally, we took a closer look at the estimators or RGANs and found that 1) the estimator of RpGANs used by \citet{RGAN} is not the minimum-variance unbiased estimator (MVUE) and 2) the estimators of RaGANs and RalfGANs are slightly biased with finite batch-sizes. Surprisingly, we found that neither using the MVUE with RpGANs or using an unbiased estimator with RaGANs and RalfGANs improved the performance. On the contrary, using better estimators always slightly decreased the quality of generated samples. This suggests that using noisy estimates of the divergences may beneficial as a regularization mechanism. This could be explained by vanishing gradients when the discriminator becomes closer to optimality \citep{GANTheorems}.

It still remains a mystery as to why RaGANs are better than RpGANs and the direct mechanism that leads to RGANs performing in a much more stable matter. Future work should attempt to better understand the effect of the critic's difference on training. Our experiments were limited to the generation of small images; thus, we encourage further experiments with the MVUE and the unbiased estimator of RaLSGAN in different settings.

\bibliographystyle{unsrtnat}

\appendix
\section{Appendices}
\addcontentsline{toc}{section}{Appendices}
\renewcommand{\thesubsection}{\Alph{subsection}}

\subsection{Sketch of the divergences proofs}

Although the four divergences have separate proofs, a similar framework is used in each proof. Each proof consists of three steps. For clarity of notation, let $D_f(\mathbb{P}, \mathbb{Q}) = \sup\limits_{C:\mathcal{X} \to \mathbb{R}} F(\mathbb{P}, \mathbb{Q}, C, f)$ be the divergence, where $F$ is any of the objective functions in Theorem 3.1.

First, we show that $D_f(\mathbb{P}, \mathbb{Q}) \ge 0$. This is easily proven by taking the simplest possible choice of critic, which does not depends on the probability distributions, i.e., $C^{w}(x)=k$ for all $x$. This critic always leads to $f(0)$ and thus to a objective function equal to 0. This means that \\ $D_f(\mathbb{P}, \mathbb{Q}) = \sup\limits_{C:\mathcal{X} \to \mathbb{R}} F(\mathbb{P}, \mathbb{Q}, C, f) \geq F(\mathbb{P}, \mathbb{Q}, C^{w}, f) = 0$.

Second, we show that $\mathbb{P} = \mathbb{Q} \implies D_f(\mathbb{P}, \mathbb{Q}) = 0$. This step generally relies on Jensen's inequality (for concave functions) which we use to show that $D_f(\mathbb{P}, \mathbb{P}) \leq 0$. Given that $D_f(\mathbb{P}, \mathbb{P}) \geq 0$ and $D_f(\mathbb{P}, \mathbb{P}) \leq 0$, we have that $D_f(\mathbb{P}, \mathbb{P}) = 0$.

Third, we show that $D_f(\mathbb{P}, \mathbb{Q}) = 0 \implies \mathbb{P} = \mathbb{Q}$. This step is by far the most difficult to prove. Instead of showing it directly, we instead prove it by contraposition, i.e., we show that $\mathbb{P} \neq \mathbb{Q} \implies  D_f(\mathbb{P}, \mathbb{Q}) > 0$. To prove this, we use the fact that if $\mathbb{P} \neq \mathbb{Q}$, there must be values of the probability density functions, $p(x)$ and $q(x)$ respectively, such that $p(x) > q(x)$ (and vice versa). Let $T=\{x|p(x)>q(x)\}$, we know that this set is not empty. To make the proof as simple as possible, we use the following sub-optimal critic:
$$C'(x) = \begin{cases}
\nabla & \text{if $x \in T$}\\
0 & \text{else},
\end{cases}$$
where $\nabla \neq 0$. This critic function is very simple, but, as we will show, there exists a $\nabla>0$ such that this leads to an objective function greater than 0 which means that the divergence is also greater than 0.

With this critic in mind, our goal is to transform the problem into the following: $$ D_f(\mathbb{P}, \mathbb{Q}) = \sup\limits_{C:\mathcal{X} \to \mathbb{R}} F(\mathbb{P}, \mathbb{Q}, C, f) \geq F(\mathbb{P}, \mathbb{Q}, C', f) \geq L(\nabla) > 0,$$ where $L(\nabla) = a f(\nabla) + b f(-\nabla)$, for some $a>0$ and $b>0$ s.t. $a>b$. We have been able to show this with all divergences.

We want to find a $\nabla > 0$ large enough so that the positive term ($f(\nabla)$) is big, but small enough so that the negative term ($f(-\nabla)$) is not too big. The main caveat is that, by concavity, $f(\nabla) \leq |f(-\nabla)|$. This means that the negative term is always bigger in absolute value than the positive term. This is problematic, since $a$ could be be very close to $b$ and we want $a f(\nabla) > b f(-\nabla)$ to get $L(\nabla)> 0$ and show that we have a divergence. The solution is to choose $\nabla$ to be very small. By continuity of the concave function, if we make $\nabla$ small enough (very close to 0), we can reach a point where $(f(\nabla) \approx -f(-\nabla))$. In which case, if $a = b + \epsilon$, we have that $$L(\nabla) = a f(\nabla) + b f(-\nabla) \approx a f(\nabla) - b f(\nabla) = b f(\nabla) + \epsilon f(\nabla) - b f(\nabla) = \epsilon f(\nabla) > 0.$$ In the actual proof, we show that there always exists a $\delta>0$ small enough such that any $\nabla \in (0,\delta)$ leads to $L(\nabla)>0$. This concludes the sketch of the proof.

\subsection{Proving that the objective functions are divergences}

\theoremstyle{definition}

\begin{definition}\label{1.0}
	Let $\mathbb{P}$ and $\mathbb{Q}$ be probability distributions and $S$ be the set of all probability distributions with common support. A function $D:(S,S) \to \mathbb{R}_{>0} $ is a divergence if it respects the following two conditions:
	\begin{align*}
	&D(\mathbb{P}, \mathbb{Q}) \ge 0 \\ 
	&D(\mathbb{P}, \mathbb{Q}) = 0 \iff \mathbb{P} = \mathbb{Q}.
	\end{align*}
\end{definition}

\begin{definition}\label{1.0}
	A function $f$ is concave on $X$ if and only if
	$$
	\forall x, y \in X, \forall \alpha \in [0,1]: f(\alpha x + (1-\alpha)y) \geq \alpha f(x) + (1-\alpha) f(y).
	$$
\end{definition}
\begin{lemma}\label{1.0}
	Let $f$ be a concave function on $X$, we have that
	$$
	\forall x_1, x_2, x_3 \in X \text{ s.t. } x_1 < x_2 \leq x_3: \frac{f(x_3)-f(x_1)}{x_3-x_1} \leq \frac{f(x_2) - f(x_1)}{(x_2-x_1)}
	$$
	and
	$$
	\forall x_1, x_2, x_3 \in X \text{ s.t. } x_1 \leq x_2 < x_3: \frac{f(x_3) - f(x_2)}{(x_3-x_2)} \leq \frac{f(x_3)-f(x_1)}{x_3-x_1}.
	$$
\end{lemma}
\begin{proof}
	$ $
	
	Let $\alpha = \frac{(x_3-x_2)}{(x_3-x_1)}$. \\
	If $x_1 < x_2 \leq x_3$, we have that $\alpha \in [0,1)$. \\
	If $x_1 \leq x_2 < x_3$, we have that $\alpha \in (0,1)$. \\
	Either way, by concavity, we have that 
	
	\begin{align*}
	f(x_2) &\geq \frac{(x_3-x_2)}{(x_3-x_1)} f(x_1) + \left( 1 - \frac{(x_3-x_2)}{(x_3-x_1)} \right) f(x_3) \\
	&= \frac{(x_3-x_2)}{(x_3-x_1)} f(x_1) + \frac{(x_2-x_1)}{(x_3-x_1)}  f(x_3) 
	\end{align*}
	
	If $x_1 < x_2 \leq x_3$, we have that:
	\begin{align*}
	f(x_2) - f(x_1) &\geq \frac{(x_1-x_2)f(x_1)+(x_2-x_1)f(x_3)}{(x_3-x_1)} \\
	\frac{f(x_2) - f(x_1)}{(x_2-x_1)} &\geq \frac{f(x_3)-f(x_1)}{(x_3-x_1)}
	\end{align*}
	
	If $x_1 \leq x_2 < x_3$, we have that:
	\begin{align*}
	f(x_2) - f(x_3) &\geq \frac{(x_3-x_2)f(x_1)+(x_2-x_3)f(x_3)}{(x_3-x_1)} \\
	\frac{f(x_2) - f(x_3)}{(x_3-x_2)} &\geq \frac{f(x_1)-f(x_3)}{(x_3-x_1)} \\
	\frac{f(x_3) - f(x_2)}{(x_3-x_2)} &\leq \frac{f(x_3)-f(x_1)}{(x_3-x_1)}
	\end{align*}
	
\end{proof}
\begin{lemma}\label{1.0}
	Let $f:\mathbb{R} \to \mathbb{R}$ be a concave function such that $f(0)=0$. We have that
	$$
	\forall a,b,\nabla \text{ s.t. } b \geq a > 0, \nabla \neq 0: \frac{f(\nabla b)}{b} \leq \frac{f(\nabla a)}{a}.
	$$
\end{lemma}
\begin{proof}
	$ $
	
	If $\nabla>0$ we have that  $0 < \nabla a \leq \nabla b$. \\
	By Lemma A.1 , we have that
	\begin{align*}
	&\frac{f(\nabla b)-f(0)}{\nabla(b-0)} \leq \frac{f(\nabla a)-f(0)}{\nabla(a-0)} \\
	\iff &\frac{f(\nabla b)}{b} \leq \frac{f(\nabla a)}{a}
	\end{align*}
	
	If $\nabla<0$, we have that  $\nabla b \leq \nabla a < 0$. \\
	By Lemma A.1, we have that
	\begin{align*}
	&\frac{f(0)-f(\nabla a)}{\nabla(0-a)} \leq \frac{f(0)-f(\nabla b)}{\nabla(0-b)} \\
	\iff &\frac{f(\nabla a)}{\nabla a} \leq \frac{f(\nabla b)}{\nabla b}  \\
	\iff &\frac{f(\nabla a)}{a} \geq \frac{f(\nabla b)}{b} \text{, since }\nabla<0
	\end{align*}
	
	Thus, when $\nabla \neq 0$, we have that 
	$$
	\frac{f(\nabla b)}{b} \leq \frac{f(\nabla a)}{a}
	$$
	
\end{proof}

\begin{lemma}
	Let $f:\mathbb{R} \to \mathbb{R}$ be a concave function such that $f(0)=0$, $f$ is differentiable at 0, $f'(0)\neq0$, $\sup_x f(x) = M > 0$, and $\argsup_x f(x)> 0$. Let $L(\nabla) = a f(\nabla) + b f(-\nabla)$, where $a>0$, $b>0$, and $a \neq b$.

	If $a > b$, $\exists \delta > 0$, s.t. $\forall \nabla^{*} \in (0,\delta): L(\nabla^{*}) > 0$ \\
	If $a < b$, $\exists \delta > 0$, s.t. $\forall \nabla^{*} \in (-\delta,0): L(\nabla^{*}) > 0$.
\end{lemma} 
\begin{proof}
	$ $
	
	By concavity, for all $\alpha \in (0,1]$, we have $f(\alpha x^{*}) \geq \alpha f(x^{*}) > 0$. \\
	This means that for any $\nabla \in (0,x^{*}]$, we have that $f(\nabla) > 0$.
	
	By concavity, for all $x$, we have that $\frac{1}{2} f(x) + \frac{1}{2} f(-x) \leq f(\frac{1}{2} x - \frac{1}{2} x) = f(0) = 0$. \\
	Thus, for all $\nabla \in (0,x^{*}]$ we have that $0 < f(\nabla) \leq -f(-\nabla)$. \\
	This means that $f(\nabla) > 0$ and $f(-\nabla) < 0$.

	Let $R(x) = \frac{g(x)}{f(x)}$, where $g(x)=-f(-x)$. \\
	We can show that:
	$$\lim\limits_{x\to 0} R(x) = \lim\limits_{x\to 0} \frac{g(x)}{f(x)} \stackrel{\text{H}}{=} \lim\limits_{x\to 0} \frac{g'(x)}{f'(x)} = \lim\limits_{x\to 0} \frac{f'(-x)}{f'(x)} = \frac{f'(0)}{f'(0)} = 1.$$
	
	If $\nabla \in (0,x^{*}]$, by concavity we have that $0 < f(\nabla) \leq -f(-\nabla)$, thus $R(\nabla) = \frac{-f(-\nabla)}{f(\nabla)} \geq 1$. \\
	
	Let $\epsilon = \frac{(a'-b')}{b'}$, where $a'>b'>0$. \\
	By the definition of the limit, $\exists \delta > 0$ s.t. $\forall x$ s.t. $0<|x|<\delta$, we have 
	$$ |R(x)-1| < \epsilon. $$
	Since this is true for all $x$ s.t. $0<|x|<\delta$, this is also true for all $0<\nabla^{*}<\min(x^{*},\delta)$. \\
	This means that
	\begin{align*}
	&|R(\nabla^{*})-1| < \epsilon \\
	\implies &(R(\nabla^{*})-1) < \frac{(a'-b')}{b'} \text{, since } R(\nabla) \geq 1 \text{ for all } \nabla \in (0,x^{*}] \\
	\implies &R(\nabla^{*}) < \frac{a'}{b'} \\
	\implies &\frac{-f(-\nabla^{*})}{f(\nabla^{*})} < \frac{a'}{b'} \\
	\implies &a' f(\nabla^{*}) + b' f(-\nabla^{*}) > 0 \\
	\end{align*}
	If $a>b$, let $a'=a$, $b'=b$,  and we have $a f(\nabla^{*}) + b f(-\nabla^{*}) > 0$ for all $0<\nabla^{*}<\min(x^{*},\delta)$. \\
	If $a<b$, let $a'=b$, $b'=a$, and we have $a f(\nabla^{*}) + b f(-\nabla^{*}) > 0$ for all $-\min(x^{*},\delta) < \nabla^{*} < 0$. \\

\end{proof}

\begin{theorem}
	Let $f:\mathbb{R} \to \mathbb{R}$ be a concave function such that $f(0)=0$, $f$ is differentiable at 0, $f'(0)\neq0$, $\sup_x f(x) = M > 0$, and $\argsup_x f(x)> 0$. Let $\mathbb{P}$ and $\mathbb{Q}$ be probability distributions with support $\mathcal{X}$. Then, we have that \[ \mathrm{D}^{Rp}_f(\mathbb{P}, \mathbb{Q}) = \sup\limits_{C:\mathcal{X} \to \mathbb{R}} \hspace{1pt}
	\underset{ \substack{x \sim \mathbb{P}  \\ y \sim \mathbb{Q}}}{\mathbb{E}\vphantom{p}} \left[ f \left( C(x) - C(y) \right) \right] \] is a divergence.
\end{theorem}
\begin{proof}
	$ $
	
	Let $C^{w}(x) = k \hspace{4pt} \forall x$ (worst possible choice of $C$). \\
	Let $C^{*}(x) = \underset{C:\mathcal{X} \to \mathbb{R}}{\argsup} \hspace{1pt}
	\underset{ \substack{x \sim \mathbb{P}  \\ y \sim \mathbb{Q}}}{\mathbb{E}\vphantom{p}} \left[ f \left( C(x) - C(y) \right) \right]$ (best possible choice of $C$).
	
	\textbf{\#1} Proof that $\mathrm{D}^{Rp}_f(\mathbb{P}, \mathbb{Q}) \geq 0$\\
	
	$\mathrm{D}^{Rp}_f(\mathbb{P}, \mathbb{Q}) = \underset{ \substack{x \sim \mathbb{P}  \\ y \sim \mathbb{Q}}}{\mathbb{E}\vphantom{p}} \left[ f \left( C^{*}(x) - C^{*}(y) \right) \right] \geq \underset{ \substack{x \sim \mathbb{P}  \\ y \sim \mathbb{Q}}}{\mathbb{E}\vphantom{p}} \left[ f \left( C^{w}(x) - C^{w}(y) \right) \right] = 0$. \\
	
	\textbf{\#2} Proof that $\mathbb{P} = \mathbb{Q} \implies \mathrm{D}^{Rp}_f(\mathbb{P}, \mathbb{Q}) = 0$
	
	\begin{align*}
	\mathrm{D}^{Rp}_f(\mathbb{P}, \mathbb{Q}) &= \underset{ \substack{x \sim \mathbb{P}  \\ y \sim \mathbb{P}}}{\mathbb{E}\vphantom{p}} \left[ f \left( C^{*}(x) - C^{*}(y) \right) \right] \\
	&= \underset{x \sim \mathbb{P}}{\mathbb{E}\vphantom{p}} \left[ \underset{y \sim \mathbb{P}}{\mathbb{E}\vphantom{p}} \left[ f \left( C^{*}(x) - C^{*}(y) \right) | x \right] \right] \\
	&\leq \underset{x \sim \mathbb{P}}{\mathbb{E}\vphantom{p}} \left[ f \left( \underset{y \sim \mathbb{P}}{\mathbb{E}\vphantom{p}} \left[ C^{*}(x) - C^{*}(y) | x \right] \right) \right] \\
	&= \underset{x \sim \mathbb{P}}{\mathbb{E}\vphantom{p}} \left[ f \left( C^{*}(x) - \underset{y \sim \mathbb{P}}{\mathbb{E}\vphantom{p}} \left[C^{*}(y) \right] \right) \right] \\
	&= \underset{\scriptstyle x \sim \mathbb{P}}{\mathbb{E}\vphantom{p}} \left[ f \left( C'^{*}(x) \right) \right]\text{, where } C'^{*}(x) = C^{*}(x) - \underset{y \sim \mathbb{P}}{\mathbb{E}\vphantom{p}} \left[C^{*}(y) \right] \\
	&\leq f \left( \underset{\scriptstyle x \sim \mathbb{P}}{\mathbb{E}\vphantom{p}} \left[ C'^{*}(x) \right] \right) \text{, by Jensen's inequality} \\
	& = f(0) \\
	& = 0
	\end{align*}
	Since $\mathrm{D}^{Rp}_f(\mathbb{P}, \mathbb{Q}) \geq 0$, we have that $\mathrm{D}^{Rp}_f(\mathbb{P}, \mathbb{Q}) = 0$.
	
	\textbf{\#3} Proof that $\mathrm{D}^{Rp}_f(\mathbb{P}, \mathbb{Q}) = 0 \implies \mathbb{P} = \mathbb{Q}$ \\
	
	We prove this by contraposition (i.e., we prove that $\mathbb{P} \neq \mathbb{Q} \implies \mathrm{D}^{Rp}_f(\mathbb{P}, \mathbb{Q}) \neq 0$). To do so, we design a function $C'$ that is better than the worse option ($C(x) = k \hspace{4pt} \forall x$).
	
	Assume that $\mathbb{P} \neq \mathbb{Q}$.
	
	Let $T = \argsup_S \mathbb{P}(S) - \mathbb{Q}(S)$ \footnote{If $\mathbb{P}$ and $\mathbb{Q}$ have probability density functions $p(x)$ and $q(x)$ respectively, then $T=\{x|p(x)>q(x)\}$.}. \\
	Let $p = \int_{T}  d\mathbb{P}(x) \implies (1-p) = \int_{\mathcal{X} \setminus T}  d\mathbb{P}(x)$. \\
	Let $q = \int_{T}  d\mathbb{Q}(y) \implies (1-q) = \int_{\mathcal{X} \setminus T}  d\mathbb{Q}(y)$. \\
	Since $\mathbb{P} \neq \mathbb{Q}$, we know that $T \neq \varnothing$. \\
	This means that $p > 0$, $q > 0$, and $p>q$. \\
	
	Let $C'(x) = \begin{cases}
	\nabla & \text{if $x \in T$}\\
	0 & \text{else}
	\end{cases}$, where $\nabla \neq 0$. \\
	Let $L(\nabla) = \underset{ \substack{x \sim \mathbb{P}  \\ y \sim \mathbb{Q}}}{\mathbb{E}\vphantom{p}} \left[ f \left( C'(x) - C'(y) \right) \right]$. \\
	
	We have that
	\begin{align*}
	L(\nabla) = & \int_{\mathcal{X}} \int_{\mathcal{X}}   f \left( C'(x) - C'(y) \right) d\mathbb{P}(x) d\mathbb{Q}(y) \\
	= & \int_{T} \int_{T}   f \left( C'(x) - C'(y) \right) d\mathbb{P}(x) d\mathbb{Q}(y) + \int_{T} \int_{\mathcal{X} \setminus T}   f \left( C'(x) - C'(y) \right) d\mathbb{P}(x) d\mathbb{Q}(y) + \hphantom{a} \\ 
	& \int_{\mathcal{X} \setminus T} \int_{T}   f \left( C'(x) - C'(y) \right) d\mathbb{P}(x) d\mathbb{Q}(y) + \int_{\mathcal{X} \setminus T} \int_{\mathcal{X} \setminus T}   f \left( C'(x) - C'(y) \right) d\mathbb{P}(x) d\mathbb{Q}(y) \\
	& = (1) + (2) + (3) + (4)
	\end{align*}	
	\begin{align*}
	&(1) \int_{T} \int_{T}   f \left( C'(x) - C'(y) \right) d\mathbb{P}(x) d\mathbb{Q}(y) = \int_{T} \int_{T}   f \left( \nabla - \nabla \right) d\mathbb{P}(x) d\mathbb{Q}(y) = 0 \\
	&(2) \int_{T} \int_{\mathcal{X} \setminus T}   f \left( C'(x) - C'(y) \right) d\mathbb{P}(x) d\mathbb{Q}(y) = f(\nabla) \int_{T}  d\mathbb{P}(x) \int_{\mathcal{X} \setminus T}  d\mathbb{Q}(y)  = f(\nabla)p(1-q) \\
	&(3) \int_{\mathcal{X} \setminus T} \int_{T}   f \left( C'(x) - C'(y) \right) d\mathbb{P}(x) d\mathbb{Q}(y) = f(-\nabla) \int_{\mathcal{X} \setminus T}  d\mathbb{P}(x) \int_{T}  d\mathbb{Q}(y)  = f(-\nabla)q(1-p) \\
	&(4) \int_{\mathcal{X} \setminus T} \int_{\mathcal{X} \setminus T}   f \left( C'(x) - C'(y) \right) d\mathbb{P}(x) d\mathbb{Q}(y) = \int_{\mathcal{X} \setminus T} \int_{\mathcal{X} \setminus T}   f \left( 0 - 0 \right) d\mathbb{P}(x) d\mathbb{Q}(y) = 0
	\end{align*}
	
	This means that $L(\nabla) = a f(\nabla) + b f(-\nabla)$, where $a = p(1-q) > 0$ and $b = q(1-p) > 0$. \\ We know that $a = p(1-q) > q(1-p) = b$. \\
	Thus, by Lemma A.4, we have that $\exists \nabla^{*}>0$ s.t. $L(\nabla^{*}) > 0$.

	Thus, if we let $\nabla = \nabla^{*}$, we have that \\
	 $\mathrm{D}^{Rp}_f(\mathbb{P}, \mathbb{Q}) = \underset{ \substack{x \sim \mathbb{P}  \\ y \sim \mathbb{Q}}}{\mathbb{E}\vphantom{p}} \left[ f \left( C^{*}(x) - C^{*}(y) \right) \right] \geq \underset{ \substack{x \sim \mathbb{P}  \\ y \sim \mathbb{Q}}}{\mathbb{E}\vphantom{p}} \left[ f \left( C'(x) - C'(y) \right) \right] > 0$.
	
\end{proof}

\begin{theorem}
	Let $f:\mathbb{R} \to \mathbb{R}$ be a concave function such that $f(0)=0$, $f$ is differentiable at 0, $f'(0)\neq0$, $\sup_x f(x) = M > 0$, and $\argsup_x f(x)> 0$. Let $\mathbb{P}$ and $\mathbb{Q}$ be probability distributions with support $\mathcal{X}$. Then, we have that \[ \mathrm{D}^{Ralf}_{f}(\mathbb{P}, \mathbb{Q}) = \sup\limits_{\scriptstyle C:\mathcal{X} \to \mathbb{R}} \hspace{1pt}
	\underset{\scriptstyle x \sim \mathbb{P}}{\mathbb{E}\vphantom{p}} \left[ f \left( C(x) - \underset{y \sim \mathbb{Q}}{\mathbb{E}}C(y) \right) \right] \] is a divergence.
\end{theorem}
\begin{proof}
	$ $
	
	Let $C^{w}(x) = k \hspace{4pt} \forall x$ (worst possible choice of $C$). \\
	Let $C^{*}(x) = \underset{C:\mathcal{X} \to \mathbb{R}}{\argsup} \hspace{1pt}
	\underset{x \sim \mathbb{P}}{\mathbb{E}\vphantom{p}} \left[ f \left( C(x) - \underset{y \sim \mathbb{Q}}{\mathbb{E}}C(y) \right) \right]$ (best possible choice of $C$).
	
	\textbf{\#1} Proof that $\mathrm{D}^{Ralf}_{f}(\mathbb{P}, \mathbb{Q}) \geq 0$\\
	\begin{align*}
		\mathrm{D}^{Ralf}_{f}(\mathbb{P}, \mathbb{Q}) &= \underset{x \sim \mathbb{P}}{\mathbb{E}\vphantom{p}} \left[ f \left( C^{*}(x) - \underset{y \sim \mathbb{Q}}{\mathbb{E}}C^{*}(y) \right) \right] \\
		&\geq \underset{x \sim \mathbb{P}}{\mathbb{E}\vphantom{p}} \left[ f \left( C^{w}(x) - \underset{y \sim \mathbb{Q}}{\mathbb{E}}C^{w}(y) \right) \right] \\
		&= \underset{x \sim \mathbb{P}}{\mathbb{E}\vphantom{p}} \left[ f \left( k - k \right) \right] \\
		&= 0.
	\end{align*}
	
	\textbf{\#2} Proof that $\mathbb{P} = \mathbb{Q} \implies \mathrm{D}^{Ralf}_{f}(\mathbb{P}, \mathbb{Q}) = 0$
	\begin{align*}
	\mathrm{D}^{Ralf}_{f}(\mathbb{P}, \mathbb{Q}) &= \sup\limits_{\scriptstyle C:\mathcal{X} \to \mathbb{R}} \hspace{1pt} \underset{\scriptstyle x \sim \mathbb{P}}{\mathbb{E}\vphantom{p}} \left[ f \left( C(x) - \underset{y \sim \mathbb{Q}}{\mathbb{E}}C(y) \right) \right] \\
	&= \sup\limits_{\scriptstyle C:\mathcal{X} \to \mathbb{R}} \hspace{1pt} \underset{\scriptstyle x \sim \mathbb{P}}{\mathbb{E}\vphantom{p}} \left[ f \left( C(x) - \underset{y \sim \mathbb{P}}{\mathbb{E}}C(y) \right) \right] \text{, since } \mathbb{P} = \mathbb{Q} \\
	&= \sup\limits_{\substack{\scriptstyle C':\mathcal{X} \to \mathbb{R} \\ \text{s.t. } \mathbb{E}[C'(x)] = 0}} \hspace{1pt} \underset{\scriptstyle x \sim \mathbb{P}}{\mathbb{E}\vphantom{p}} \left[ f \left( C'(x) \right) \right] \\
	&= \underset{\scriptstyle x \sim \mathbb{P}}{\mathbb{E}\vphantom{p}} \left[ f \left( C'^{*}(x) \right) \right] \text{, where } C'^{*} = \argsup\limits_{\substack{\scriptstyle C':\mathcal{X} \to \mathbb{R} \\ \text{s.t. } \mathbb{E}[C'(x)] = 0}} \hspace{1pt}
	\underset{x \sim \mathbb{P}}{\mathbb{E}\vphantom{p}} \left[ f \left( C'(x) \right) \right] \\
	&\leq f \left( \underset{\scriptstyle x \sim \mathbb{P}}{\mathbb{E}\vphantom{p}} \left[ C'^{*}(x) \right] \right) \text{, by Jensen's inequality} \\
	& = f(0) \\
	& = 0
	\end{align*}
	
	Since $\mathrm{D}^{Ralf}_{f}(\mathbb{P}, \mathbb{Q}) \geq 0$, we have that $\mathrm{D}^{Ralf}_{f}(\mathbb{P}, \mathbb{Q}) = 0$.
	
	\textbf{\#3} Proof that $\mathrm{D}^{Ra}_{f}(\mathbb{P}, \mathbb{Q}) = 0 \implies \mathbb{P} = \mathbb{Q}$ \\
	
	We prove this by contraposition (i.e., we prove that $\mathbb{P} \neq \mathbb{Q} \implies \mathrm{D}^{Ra}_{f}(\mathbb{P}, \mathbb{Q}) \neq 0$). To do so, we design a function $C'$ that is better than the worse option ($C(x) = k \hspace{4pt} \forall x$).
	
	Assume that $\mathbb{P} \neq \mathbb{Q}$.
	
	Let $T = \argsup_S \mathbb{P}(S) - \mathbb{Q}(S)$.\\
	Let $p = \int_{T}  d\mathbb{P}(x) \implies (1-p) = \int_{\mathcal{X} \setminus T}  d\mathbb{P}(x)$. \\
	Let $q = \int_{T}  d\mathbb{Q}(y) \implies (1-q) = \int_{\mathcal{X} \setminus T}  d\mathbb{Q}(y)$. \\
	
	Since $\mathbb{P} \neq \mathbb{Q}$, we know that $T \neq \varnothing$. \\
	This means that $p > 0$, $q > 0$, and $p>q$. \\
	
	Let $C'(x) = \begin{cases}
	\nabla & \text{if $x \in T$}\\
	0 & \text{else}
	\end{cases}$, where $\nabla \neq 0$. \\
	Let $L(\nabla) = \underset{\scriptstyle x \sim \mathbb{P}}{\mathbb{E}\vphantom{p}} \left[ f \left( C'(x) - \underset{y \sim \mathbb{Q}}{\mathbb{E}}C'(y) \right) \right]$. \\
	
	We have that
	\begin{align*}
	L(\nabla) = & \int_{\mathcal{X}}  f \left( C'(x) - \underset{y \sim \mathbb{Q}}{\mathbb{E}}C'(y) \right) d\mathbb{P}(x) \\
	= & \int_{\mathcal{X}}  f \left( C'(x) - \int_{T} \nabla d\mathbb{Q}(y) \right) d\mathbb{P}(x) \\
	= & \int_{\mathcal{X}}  f \left( C'(x) - \nabla q \right) d\mathbb{P}(x) \\
	= & \int_{T}  f \left( \nabla - \nabla q \right) d\mathbb{P}(x) + \int_{\mathcal{X} \setminus T}  f \left( 0 - \nabla q \right) d\mathbb{P}(x) \\
	= & \hspace*{2pt} p f \left( \nabla (1 - q) \right) + (1-p) f \left(- \nabla q \right)
	\end{align*}
	
	Case 1: If $q < (1-q)$, by Lemma A.3, we have that:
	\begin{align*}
	&\frac{f(-\nabla (1-q))}{(1-q)} \leq \frac{f(-\nabla q)}{q} \\
	\implies &f(-\nabla q) \geq \frac{q}{(1-q)}f(-\nabla (1-q))
	\end{align*}
	Thus, $L(\nabla) \geq p f \left( \nabla (1 - q) \right) + \frac{(1-p)q}{(1-q)} f \left(-\nabla (1 - q) \right)$. \\
	Knowing that $p > q$ and $(1-p) < (1-q)$, we have that $p > q > \frac{q(1-p)}{(1-q)}$.\\
	Thus, by Lemma A.4, we have that $\exists \nabla^{*}>0$ s.t. $L(\nabla^{*}) > 0$. \\
	
	Case 2: If $q \geq (1-q)$, by Lemma A.3, we have that:
	\begin{align*}
	&\frac{f(\nabla q)}{q} \leq \frac{f(\nabla (1-q))}{(1-q)} \\
	\implies &f(\nabla (1-q)) \geq \frac{(1-q)}{q}f(\nabla q)
	\end{align*}
	Thus, $L(\nabla) \geq \frac{p(1-q)}{q} f \left( \nabla q \right) + (1-p) f \left(-\nabla q \right)$. \\
	Knowing that $p > q$ and $(1-p) < (1-q)$, we have that $(1-p) < (1-q) < \frac{(1-q)p}{q}$.\\
	Thus, by Lemma A.4, we have that $\exists \nabla^{*}>0$ s.t. $L(\nabla^{*}) > 0$. \\
	
	Thus, if we let $\nabla = \nabla^{*}$, we have that \\
	$\mathrm{D}^{Ralf}_{f}(\mathbb{P}, \mathbb{Q}) = \underset{x \sim \mathbb{P}}{\mathbb{E}\vphantom{p}} \left[ f \left( C^{*}(x) - \underset{y \sim \mathbb{Q}}{\mathbb{E}}C^{*}(y) \right) \right] \geq \underset{x \sim \mathbb{P}}{\mathbb{E}\vphantom{p}} \left[ f \left( C'(x) - \underset{y \sim \mathbb{Q}}{\mathbb{E}}C'(y) \right) \right] > 0$.
	
\end{proof}

\begin{theorem}
	Let $f:\mathbb{R} \to \mathbb{R}$ be a concave function such that $f(0)=0$, $f$ is differentiable at 0, $f'(0)\neq0$, $\sup_x f(x) = M > 0$, and $\argsup_x f(x)> 0$. Let $\mathbb{P}$ and $\mathbb{Q}$ be probability distributions with support $\mathcal{X}$. Then, we have that \[ \mathrm{D}^{Ra}_{f}(\mathbb{P}, \mathbb{Q}) = \sup\limits_{\scriptstyle C:\mathcal{X} \to \mathbb{R}} \hspace{1pt}
	\underset{\scriptstyle x \sim \mathbb{P}}{\mathbb{E}\vphantom{p}} \left[ f \left( C(x) - \underset{y \sim \mathbb{Q}}{\mathbb{E}}C(y) \right) \right] +
	\underset{\scriptstyle y \sim \mathbb{Q}}{\mathbb{E}\vphantom{p}} \left[ f \left( \underset{x \sim \mathbb{P\vphantom{Q}}}{\mathbb{E}}C(x) - C(y) \right) \right] \] is a divergence.
\end{theorem}
\begin{proof}
	$ $
	
	Let $C^{w}(x) = k \hspace{4pt} \forall x$ (worst possible choice of $C$). \\
	Let $C^{*}(x) = \underset{C:\mathcal{X} \to \mathbb{R}}{\argsup} \hspace{1pt}
	\underset{x \sim \mathbb{P}}{\mathbb{E}\vphantom{p}} \left[ f \left( C(x) - \underset{y \sim \mathbb{Q}}{\mathbb{E}}C(y) \right) \right] +	\underset{\scriptstyle y \sim \mathbb{Q}}{\mathbb{E}\vphantom{p}} \left[ f \left( \underset{x \sim \mathbb{P\vphantom{Q}}}{\mathbb{E}}C(x) - C(y) \right) \right]$ \\ (best possible choice of $C$).
	
	\textbf{\#1} Proof that $\mathrm{D}^{Ra}_{f}(\mathbb{P}, \mathbb{Q}) \geq 0$
	\begin{align*}
	\mathrm{D}^{Ra}_{f}(\mathbb{P}, \mathbb{Q}) &= \underset{x \sim \mathbb{P}}{\mathbb{E}\vphantom{p}} \left[ f \left( C^{*}(x) - \underset{y \sim \mathbb{Q}}{\mathbb{E}}C^{*}(y) \right) \right] + \underset{\scriptstyle y \sim \mathbb{Q}}{\mathbb{E}\vphantom{p}} \left[ f \left( \underset{x \sim \mathbb{P\vphantom{Q}}}{\mathbb{E}}C^{*}(x) - C^{*}(y) \right) \right] \\
	&\geq \underset{x \sim \mathbb{P}}{\mathbb{E}\vphantom{p}} \left[ f \left( C^{w}(x) - \underset{y \sim \mathbb{Q}}{\mathbb{E}}C^{w}(y) \right) \right] + \underset{\scriptstyle y \sim \mathbb{Q}}{\mathbb{E}\vphantom{p}} \left[ f \left( \underset{x \sim \mathbb{P\vphantom{Q}}}{\mathbb{E}}C^{w}(x) - C^{w}(y) \right) \right] \\
	&= \underset{x \sim \mathbb{P}}{\mathbb{E}\vphantom{p}} \left[ f \left( k - k \right) \right] + \underset{x \sim \mathbb{Q}}{\mathbb{E}\vphantom{p}} \left[ f \left( k - k \right) \right] \\
	&= 0.
	\end{align*}
	
	\textbf{\#2} Proof that $\mathbb{P} = \mathbb{Q} \implies \mathrm{D}^{Ra}_{f}(\mathbb{P}, \mathbb{Q}) = 0$
	
	Let $C'(x) = C(x)- \underset{x \sim \mathbb{P\vphantom{Q}}}{\mathbb{E}}C(x)$
	\begin{align*}
		\mathrm{D}^{Ra}_{f}(\mathbb{P}, \mathbb{Q}) &= \sup\limits_{\scriptstyle C:\mathcal{X} \to \mathbb{R}} \hspace{1pt}
		\underset{\scriptstyle x \sim \mathbb{P}}{\mathbb{E}\vphantom{p}} \left[ f \left( C(x) - \underset{y \sim \mathbb{Q}}{\mathbb{E}}C(y) \right) \right] +
		\underset{\scriptstyle y \sim \mathbb{Q}}{\mathbb{E}\vphantom{p}} \left[ f \left( \underset{x \sim \mathbb{P\vphantom{Q}}}{\mathbb{E}}C(x) - C(y) \right) \right] \\
		& = \sup\limits_{\scriptstyle C:\mathcal{X} \to \mathbb{R}} \hspace{1pt}
		\underset{\scriptstyle x \sim \mathbb{P}}{\mathbb{E}\vphantom{p}} \left[ f \left( C(x) - \underset{y \sim \mathbb{P}}{\mathbb{E}}C(y) \right) \right] +
		\underset{\scriptstyle x \sim \mathbb{P}}{\mathbb{E}\vphantom{p}} \left[ f \left( \underset{x \sim \mathbb{P\vphantom{Q}}}{\mathbb{E}}C(y) - C(x) \right) \right] \\
		& = \sup\limits_{\substack{\scriptstyle C':\mathcal{X} \to \mathbb{R} \\ \text{s.t. } \mathbb{E}[C'(x)] = 0}} \hspace{1pt}
		\underset{\scriptstyle x \sim \mathbb{P}}{\mathbb{E}\vphantom{p}} \left[ f \left( C'(x) \right) + f \left( - C'(x) \right) \right] \\
		& \leq 2 \sup\limits_{\substack{\scriptstyle C':\mathcal{X} \to \mathbb{R} \\ \text{s.t. } \mathbb{E}[C'(x)] = 0}} \hspace{1pt}
		\underset{\scriptstyle x \sim \mathbb{P}}{\mathbb{E}\vphantom{p}} \left[ f \left( \frac{1}{2} C'(x) - \frac{1}{2} C'(x) \right) \right] \text{, by concavity} \\
		& = 2 \sup\limits_{\substack{\scriptstyle C':\mathcal{X} \to \mathbb{R} \\ \text{s.t. } \mathbb{E}[C'(x)] = 0}} \hspace{1pt}
		\underset{\scriptstyle x \sim \mathbb{P}}{\mathbb{E}\vphantom{p}} \left[ f \left( 0 \right) \right] \\
		& = 0
	\end{align*}
	Since $\mathrm{D}^{Ra}_{f}(\mathbb{P}, \mathbb{Q}) \geq 0$, we have that $\mathrm{D}^{Ra}_{f}(\mathbb{P}, \mathbb{Q}) = 0$.
	
	\textbf{\#3} Proof that $\mathrm{D}^{Ra}_{f}(\mathbb{P}, \mathbb{Q}) = 0 \implies \mathbb{P} = \mathbb{Q}$ \\
	
	We prove this by contraposition (i.e., we prove that $\mathbb{P} \neq \mathbb{Q} \implies \mathrm{D}^{Ra}_{f}(\mathbb{P}, \mathbb{Q}) \neq 0$). To do so, we design a function $C'$ that is better than the worse option ($C(x) = k \hspace{4pt} \forall x$).

	Assume that $\mathbb{P} \neq \mathbb{Q}$.
	
	Let $T = \argsup_S \mathbb{P}(S) - \mathbb{Q}(S)$.\\
	Let $p = \int_{T}  d\mathbb{P}(x) \implies (1-p) = \int_{\mathcal{X} \setminus T}  d\mathbb{P}(x)$. \\
	Let $q = \int_{T}  d\mathbb{Q}(y) \implies (1-q) = \int_{\mathcal{X} \setminus T}  d\mathbb{Q}(y)$. \\
	
	Since $\mathbb{P} \neq \mathbb{Q}$, we know that $T \neq \varnothing$. \\
	This means that $p > 0$, $q > 0$, and $p>q$. \\
	
	Let $C'(x) = \begin{cases}
	\nabla & \text{if $x \in T$}\\
	0 & \text{else}
	\end{cases}$, where $\nabla \neq 0$. \\
	Let $L(\nabla) = \underset{\scriptstyle x \sim \mathbb{P}}{\mathbb{E}\vphantom{p}} \left[ f \left( C'(x) - \underset{y \sim \mathbb{Q}}{\mathbb{E}}C'(y) \right) \right] +
	\underset{\scriptstyle y \sim \mathbb{Q}}{\mathbb{E}\vphantom{p}} \left[ f \left( \underset{x \sim \mathbb{P\vphantom{Q}}}{\mathbb{E}}C'(x) - C'(y) \right) \right] $. \\
	
	We have that
	\begin{align*}
	L(\nabla) = & \int_{\mathcal{X}}  f \left( C'(x) - \underset{y \sim \mathbb{Q}}{\mathbb{E}}C'(y) \right) d\mathbb{P}(x) + \int_{\mathcal{X}}  f \left( \underset{x \sim \mathbb{P\vphantom{Q}}}{\mathbb{E}}C'(x) - C'(y) \right) d\mathbb{Q}(y) \\
	= & \int_{\mathcal{X}}  f \left( C'(x) - \int_{T} \nabla d\mathbb{Q}(y) \right) d\mathbb{P}(x) + \int_{\mathcal{X}}  f \left( \int_{T} \nabla d\mathbb{P}(x) - C'(y) \right) d\mathbb{Q}(y) \\
	= & \int_{\mathcal{X}}  f \left( C'(x) - \nabla q \right) d\mathbb{P}(x) + \int_{\mathcal{X}}  f \left( \nabla p - C'(y) \right) d\mathbb{Q}(y) \\
	= & \int_{T}  f \left( \nabla (1-q) \right) d\mathbb{P}(x) + \int_{\mathcal{X} \setminus T}  f \left( -\nabla q \right) d\mathbb{P}(x) + \hphantom{a} \\
	& \int_{T}  f \left( \nabla (p-1) \right) d\mathbb{Q}(y) + \int_{T}  f \left( \nabla p \right) d\mathbb{Q}(y) \\
	= & \hspace*{2pt} p f \left( \nabla (1-q) \right) + (1-p) f \left( -\nabla q \right) + q f \left( \nabla (p-1) \right) + (1-q) f \left( \nabla p \right) \\
	= & \hspace*{2pt} p f \left( \nabla (1-q) \right) + (1-p) f \left( -\nabla q \right) + q f \left( -\nabla (1-p) \right) + (1-q) f \left( \nabla p \right)
	\end{align*}
	
	Case 1: If $(1-q) \geq p$, by Lemma A.3, we have that:
	\begin{align*}
	&\frac{f(\nabla (1-q))}{(1-q)} \leq \frac{f(\nabla p)}{p} \\
	\implies &f(\nabla p) \geq \frac{p}{(1-q)}f(\nabla (1-q))
	\end{align*}
	Also, we have that $(1-p) \geq q$, thus, by Lemma A.3, we have that:
	\begin{align*}
	&\frac{f(-\nabla (1-p))}{(1-p)} \leq \frac{f(-\nabla q)}{q} \\
	\implies &f(-\nabla q) \geq \frac{q}{(1-p)}f(-\nabla (1-p))
	\end{align*}
	Also, $q < p \implies (1-q) > (1-p)$, thus, by Lemma A.3, we have that:
	\begin{align*}
	&\frac{f(-\nabla (1-q))}{(1-q)} \leq \frac{f(-\nabla (1-p))}{(1-p)} \\
	\implies &f(-\nabla (1-p)) \geq \frac{(1-p)}{(1-q)}f(-\nabla (1-q))
	\end{align*}
	Thus, 
	\begin{align*}
	L(\nabla) &= \hspace*{2pt} p f \left( \nabla (1-q) \right) + (1-p) f \left( -\nabla q \right) + q f \left( -\nabla (1-p) \right) + (1-q) f \left( \nabla p \right) \\
	&\geq p f \left( \nabla (1-q) \right) + q f \left( -\nabla (1-p) \right) + q f \left( -\nabla (1-p) \right) + p f \left( \nabla (1-q) \right) \\
	&= \hspace*{2pt} 2p f \left( \nabla (1-q) \right) + 2q f \left( -\nabla (1-p) \right) \\
	&\geq 2p f \left( \nabla (1-q) \right) + 2\frac{q(1-p)}{(1-q)} f \left( -\nabla (1-q) \right)
	\end{align*}
	Knowing that $p > q$ and $(1-p) < (1-q)$, we have that $2p > 2q > \frac{2q(1-p)}{(1-q)}$.\\
	Thus, by Lemma A.4, we have that $\exists \nabla^{*}>0$ s.t. $L(\nabla^{*}) > 0$. \\
	
	Case 2: If $p > (1-q)$, by Lemma A.3, we have that:
	\begin{align*}
	&\frac{f(\nabla p)}{p} \leq \frac{f(\nabla (1-q))}{(1-q)} \\
	\implies &f(\nabla (1-q)) \geq \frac{(1-q)}{p}f(\nabla p)
	\end{align*}
	Also, we have that $q > (1-p)$, thus, by Lemma A.3, we have that:
	\begin{align*}
	&\frac{f(-\nabla q)}{q} \leq \frac{f(-\nabla (1-p))}{(1-p)} \\
	\implies &f(-\nabla (1-p)) \geq \frac{(1-p)}{q}f(-\nabla q)
	\end{align*}
	Also, $p > q$, thus, by Lemma A.3, we have that:
	\begin{align*}
	&\frac{f(-\nabla p)}{p} \leq \frac{f(-\nabla q)}{q} \\
	\implies &f(-\nabla q) \geq \frac{q}{p}f(-\nabla p)
	\end{align*}
	Thus, 
	\begin{align*}
	L(\nabla) &= \hspace*{2pt} p f \left( \nabla (1-q) \right) + (1-p) f \left( -\nabla q \right) + q f \left( -\nabla (1-p) \right) + (1-q) f \left( \nabla p \right) \\
	&\geq (1-q) f \left( \nabla p \right) + (1-p) f \left( -\nabla q \right) + (1-p) f \left( -\nabla q \right) + (1-q) f \left( \nabla p \right) \\
	&= \hspace*{2pt} 2(1-q) f \left( \nabla p \right) + 2(1-p) f \left( -\nabla q \right) \\
	&\geq 2(1-q) f \left( \nabla p \right) + 2 \frac{q(1-p)}{p} f \left( -\nabla p \right)
	\end{align*}
	Knowing that $p > q$ and $(1-p) < (1-q)$, we have that $2(1-q) > 2(1-p) > 2\frac{q(1-p)}{p}$.\\
	Thus, by Lemma A.4, we have that $\exists \nabla^{*}>0$ s.t. $L(\nabla^{*}) > 0$. \\
	
	Thus, if we let $\nabla = \nabla^{*}$, we have that
	\begin{align*}
	\mathrm{D}^{Ra}_{f}(\mathbb{P}, \mathbb{Q}) &= \underset{\scriptstyle x \sim \mathbb{P}}{\mathbb{E}\vphantom{p}} \left[ f \left( C^{*}(x) - \underset{y \sim \mathbb{Q}}{\mathbb{E}}C^{*}(y) \right) \right] +
	\underset{\scriptstyle y \sim \mathbb{Q}}{\mathbb{E}\vphantom{p}} \left[ f \left( \underset{x \sim \mathbb{P\vphantom{Q}}}{\mathbb{E}}C^{*}(x) - C^{*}(y) \right) \right] \\
	&\geq \underset{\scriptstyle x \sim \mathbb{P}}{\mathbb{E}\vphantom{p}} \left[ f \left( C'(x) - \underset{y \sim \mathbb{Q}}{\mathbb{E}}C'(y) \right) \right] +
	\underset{\scriptstyle y \sim \mathbb{Q}}{\mathbb{E}\vphantom{p}} \left[ f \left( \underset{x \sim \mathbb{P\vphantom{Q}}}{\mathbb{E}}C'(x) - C'(y) \right) \right] \\
	&> 0.
	\end{align*}

\end{proof}

\begin{theorem}
	Let $f:\mathbb{R} \to \mathbb{R}$ be a concave function such that $f(0)=0$, $f$ is differentiable at 0, $f'(0)\neq0$, $\sup_x f(x) = M > 0$, and $\argsup_x f(x)> 0$. Let $\mathbb{P}$ and $\mathbb{Q}$ be probability distributions with support $\mathcal{X}$. Let  $\mathbb{M} = \frac{1}{2}\mathbb{P} + \frac{1}{2}\mathbb{Q}$ Then, we have that \[ \mathrm{D}^{Rc}_{f}(\mathbb{P}, \mathbb{Q}) = \sup\limits_{\scriptstyle C:\mathcal{X} \to \mathbb{R}} \hspace{1pt}
	\underset{\scriptstyle x \sim \mathbb{P}}{\mathbb{E}\vphantom{p}} \left[ f \left( C(x) - \underset{m \sim \mathbb{M}}{\mathbb{E}}C(m) \right) \right] +
	\underset{\scriptstyle y \sim \mathbb{Q}}{\mathbb{E}\vphantom{p}} \left[ f \left( \underset{m \sim \mathbb{M}}{\mathbb{E}}C(m) - C(y) \right) \right] \] is a divergence.
\end{theorem}
\begin{proof}
	$ $
	
	Let $C^{w}(x) = k \hspace{4pt} \forall x$ (worst possible choice of $C$). \\
	Let $C^{*}(x) = \underset{C:\mathcal{X} \to \mathbb{R}}{\argsup} \hspace{1pt}
	\underset{\scriptstyle x \sim \mathbb{P}}{\mathbb{E}\vphantom{p}} \left[ f \left( C(x) - \underset{m \sim \mathbb{M}}{\mathbb{E}}C(m) \right) \right] +
	\underset{\scriptstyle y \sim \mathbb{Q}}{\mathbb{E}\vphantom{p}} \left[ f \left( \underset{m \sim \mathbb{M}}{\mathbb{E}}C(m) - C(y) \right) \right]$ \\ (best possible choice of $C$).
	
	\textbf{\#1} Proof that $\mathrm{D}^{Rc}_{f}(\mathbb{P}, \mathbb{Q}) \geq 0$
	
	Same proof as theorem A.6 \#1.
	
	\textbf{\#2} Proof that $\mathbb{P} = \mathbb{Q} \implies \mathrm{D}^{Rc}_{f}(\mathbb{P}, \mathbb{Q}) = 0$
	
	Same proof as theorem A.6 \#2.
	
	\textbf{\#3} Proof that $\mathrm{D}^{Rc}_{f}(\mathbb{P}, \mathbb{Q}) = 0 \implies \mathbb{P} = \mathbb{Q}$
	
	We prove this by contraposition (i.e., we prove that $\mathbb{P} \neq \mathbb{Q} \implies \mathrm{D}^{Rc}_{f}(\mathbb{P}, \mathbb{Q}) \neq 0$). To do so, we design a function $C'$ that is better than the worse option ($C(x) = k \hspace{4pt} \forall x$).
	
	Assume that $\mathbb{P} \neq \mathbb{Q}$.\\
	
	Make the same assumptions as theorem A.6 \#2. The only thing that changes is $L(\nabla)$.
	
	We instead have that
	\begin{align*} L(\nabla) &= \hspace*{2pt} p f \left( \nabla (1-c) \right) + (1-p) f \left( -\nabla c \right) + q f \left( -\nabla (1-c) \right) + (1-q) f \left( \nabla c \right) \\
	&= L_1(\nabla) + L_2(\nabla),
	\end{align*}
	where $c = \frac{1}{2}p + \frac{1}{2}q$, \\ $L_1(\nabla)=p f \left( \nabla (1-c) \right) + q f \left( -\nabla (1-c) \right)$, \\ $L_2(\nabla)=(1-q) f \left( \nabla c \right) + (1-p) f \left( -\nabla c \right)$.
	
	Knowing that $p > q$ and $(1-q) > (1-p)$, we can use Lemma A.4 to show that \\
	$\exists \delta_1 > 0$, s.t. $\forall \nabla_1^{*} \in (0,\delta_1): L_1(\nabla_1^{*}) > 0$ and $\exists \delta_2 > 0$, s.t. $\forall \nabla_2^{*} \in (0,\delta_2): L_2(\nabla_2^{*}) > 0$. \\
	Thus, let $\delta = \min(\delta_1, \delta_2)$. We have that $\forall \nabla^{*} \in (0,\delta): L_1(\nabla^{*}) > 0$ and $L_2(\nabla^{*}) > 0$. \\
	This means that $L(\nabla)=L_1(\nabla^{*})+L_2(\nabla^{*}) > 0$
	
\end{proof}

\subsection{Inequalities between Relativistic Divergences}

To prove that $\mathrm{D_1}$ is weaker than $\mathrm{D_2}$, we can just show that 
$\mathrm{D_1}(\mathbb{P}, \mathbb{Q}) \leq \mathrm{D_2}(\mathbb{P}, \mathbb{Q})$ since we have that: $$\mathrm{D_1}(\mathbb{P}_n, \mathbb{P}) \leq \mathrm{D_2}(\mathbb{P}_n, \mathbb{P})\to 0 \implies \mathrm{D_1}(\mathbb{P}_n, \mathbb{P}) \to 0.$$

\begin{theorem}
	Let $f:\mathbb{R} \to \mathbb{R}$ be a concave function such that $f(0)=0$, $f$ is differentiable at 0, $f'(0)\neq0$, $\sup_x f(x) = M > 0$, and $\argsup_x f(x)> 0$. Let $\mathbb{P}$ and $\mathbb{Q}$ be probability distributions with support $\mathcal{X}$. Then, we have that 
	\begin{itemize}
		\item $\mathrm{D}^{S}(\mathbb{P}, \mathbb{Q}) \leq \mathrm{D}^{Rp}_f(\mathbb{P}, \mathbb{Q})$ 
		\item $\mathrm{D}^{Rp}_f(\mathbb{P}, \mathbb{Q}) \leq \mathrm{D}^{Ralf}_{f}(\mathbb{P}, \mathbb{Q})$ and $\mathrm{D}^{Rp}_f(\mathbb{P}, \mathbb{Q}) \leq \mathrm{D}_f^{Ra}(\mathbb{P}, \mathbb{Q})$
	\end{itemize}
\end{theorem}
\begin{proof}
	$ $
	
	Showing that $\mathrm{D}^{S}(\mathbb{P}, \mathbb{Q}) \leq \mathrm{D}^{Rp}_f(\mathbb{P}, \mathbb{Q})$: \\
	Let $$C_S^{*}(x) = \argsup_{C: \mathcal{X} \to \mathbb{R}} \mathbb{E}_{x \sim \mathbb{P}}\left[ f(C(x)) \right] + \mathbb{E}_{z \sim \mathbb{Q}} \left[ f(-C(y)) \right]$$ \\
	and $$C_{Rp}^{*}(x) = \argsup_{C: \mathcal{X} \to \mathbb{R}} \underset{ \substack{x \sim \mathbb{P}  \\ y \sim \mathbb{Q}}}{\mathbb{E}\vphantom{p}} \left[ f \left( C(x) - C(y) \right) \right].$$
	\begin{align*}
	\mathrm{D}^{S}(\mathbb{P}, \mathbb{Q}) &= \sup_{C: \mathcal{X} \to \mathbb{R}} \mathbb{E}_{x \sim \mathbb{P}}\left[ f(C(x)) \right] + \mathbb{E}_{z \sim \mathbb{Q}} \left[ f(-C(y)) \right] \\
	&= 2 \underset{ \substack{x \sim \mathbb{P}  \\ y \sim \mathbb{Q}}}{\mathbb{E}\vphantom{p}} \left[\frac{1}{2} f \left( C_S^{*}(x) \right) + \frac{1}{2} f \left( -C_S^{*}(y) \right) \right] \\
	&\leq 2 \underset{ \substack{x \sim \mathbb{P}  \\ y \sim \mathbb{Q}}}{\mathbb{E}\vphantom{p}} \left[ f \left( \frac{1}{2} C_S^{*}(x) - \frac{1}{2} C_S^{*}(y) \right) \right] \\
	&= 2 \underset{ \substack{x \sim \mathbb{P}  \\ y \sim \mathbb{Q}}}{\mathbb{E}\vphantom{p}} \left[ f \left( C'(x) - C'(y) \right) \right] \text{, where $C'(x) = \frac{1}{2}C_S^{*}(x)$} \\
	&\leq \sup_{C: \mathcal{X} \to \mathbb{R}} 2 \underset{ \substack{x \sim \mathbb{P}  \\ y \sim \mathbb{Q}}}{\mathbb{E}\vphantom{p}} \left[ f \left( C(x) - C(y) \right) \right] \\
	&= \mathrm{D}^{Rp}_f(\mathbb{P}, \mathbb{Q})
	\end{align*}
	Showing that $\mathrm{D}^{Rp}_f(\mathbb{P}, \mathbb{Q}) \leq \mathrm{D}^{Ralf}_{f}(\mathbb{P}, \mathbb{Q})$: \\
	\begin{align*}
	\mathrm{D}^{Rp}_f(\mathbb{P}, \mathbb{Q}) &=\argsup_{C: \mathcal{X} \to \mathbb{R}} 2 \underset{ \substack{x \sim \mathbb{P}  \\ y \sim \mathbb{Q}}}{\mathbb{E}\vphantom{p}} \left[ f \left( C(x) - C(y) \right) \right]  \\
	&= 2 \underset{ \substack{x \sim \mathbb{P}  \\ y \sim \mathbb{Q}}}{\mathbb{E}\vphantom{p}} \left[ f \left( C_{Rp}^{*}(x) - C_{Rp}^{*}(y) \right) \right] \\
	&= 2 \underset{x \sim \mathbb{P}}{\mathbb{E}\vphantom{p}} \left[ \underset{y \sim \mathbb{Q}}{\mathbb{E}\vphantom{p}} \left[ f \left( C_{Rp}^{*}(x) - C_{Rp}^{*}(y) \right) | x \right] \right] \\
	&\leq 2 \underset{x \sim \mathbb{P}}{\mathbb{E}\vphantom{p}} \left[ f \left( \underset{y \sim \mathbb{Q}}{\mathbb{E}\vphantom{p}} \left[ C_{Rp}^{*}(x) - C_{Rp}^{*}(y) | x \right] \right) \right] \\
	&= 2 \underset{x \sim \mathbb{P}}{\mathbb{E}\vphantom{p}} \left[ f \left( C_{Rp}^{*}(x) - \underset{y \sim \mathbb{Q}}{\mathbb{E}\vphantom{p}} \left[C_{Rp}^{*}(y) \right] \right) \right] \\
	&\leq \sup_{C: \mathcal{X} \to \mathbb{R}} 2 \underset{x \sim \mathbb{P}}{\mathbb{E}\vphantom{p}} \left[ f \left( C(x) - \underset{y \sim \mathbb{Q}}{\mathbb{E}\vphantom{p}} \left[C(y) \right] \right) \right] \\
	& = \mathrm{D}^{Ralf}_{f}(\mathbb{P}, \mathbb{Q})
	\end{align*}
	Showing that $\mathrm{D}^{Rp}_f(\mathbb{P}, \mathbb{Q}) \leq \mathrm{D}^{Ra}_f(\mathbb{P}, \mathbb{Q})$: \\
	\begin{align*}
	\mathrm{D}^{Rp}_f(\mathbb{P}, \mathbb{Q}) &=\argsup_{C: \mathcal{X} \to \mathbb{R}} 2 \underset{ \substack{x \sim \mathbb{P}  \\ y \sim \mathbb{Q}}}{\mathbb{E}\vphantom{p}} \left[ f \left( C(x) - C(y) \right) \right]  \\
	&= 2 \underset{ \substack{x \sim \mathbb{P}  \\ y \sim \mathbb{Q}}}{\mathbb{E}\vphantom{p}} \left[ f \left( C_{Rp}^{*}(x) - C_{Rp}^{*}(y) \right) \right] \\
	&= \underset{x \sim \mathbb{P}}{\mathbb{E}\vphantom{p}} \left[ \underset{y \sim \mathbb{Q}}{\mathbb{E}\vphantom{p}} \left[ f \left( C_{Rp}^{*}(x) - C_{Rp}^{*}(y) \right) | x \right] \right] + \underset{y \sim \mathbb{Q}}{\mathbb{E}\vphantom{p}} \left[ \underset{x \sim \mathbb{P}}{\mathbb{E}\vphantom{p}} \left[ f \left( C_{Rp}^{*}(x) - C_{Rp}^{*}(y) \right) | y \right] \right] \\
	&\leq \underset{x \sim \mathbb{P}}{\mathbb{E}\vphantom{p}} \left[ f \left( \underset{y \sim \mathbb{Q}}{\mathbb{E}\vphantom{p}} \left[ C_{Rp}^{*}(x) - C_{Rp}^{*}(y) | x \right] \right) \right] + \underset{y \sim \mathbb{Q}}{\mathbb{E}\vphantom{p}} \left[ f \left( \underset{x \sim \mathbb{P}}{\mathbb{E}\vphantom{p}} \left[ C_{Rp}^{*}(x) - C_{Rp}^{*}(y) | y \right] \right) \right] \\
	&= \underset{x \sim \mathbb{P}}{\mathbb{E}\vphantom{p}} \left[ f \left( C_{Rp}^{*}(x) - \underset{y \sim \mathbb{Q}}{\mathbb{E}\vphantom{p}} \left[C_{Rp}^{*}(y) \right] \right) \right] + \underset{y \sim \mathbb{Q}}{\mathbb{E}\vphantom{p}} \left[ f \left( \underset{x \sim \mathbb{P}}{\mathbb{E}\vphantom{p}} \left[C_{Rp}^{*}(x)\right] - C_{Rp}^{*}(y) \right) \right] \\
	&\leq \sup_{C: \mathcal{X} \to \mathbb{R}} \underset{x \sim \mathbb{P}}{\mathbb{E}\vphantom{p}} \left[ f \left( C(x) - \underset{y \sim \mathbb{Q}}{\mathbb{E}\vphantom{p}} \left[C(y) \right] \right) \right] + \underset{y \sim \mathbb{Q}}{\mathbb{E}\vphantom{p}} \left[ f \left( \underset{x \sim \mathbb{P}}{\mathbb{E}\vphantom{p}} \left[C(x) \right] - C(y) \right) \right] \\
	& = \mathrm{D}_f^{Ra}(\mathbb{P}, \mathbb{Q})
	\end{align*}
\end{proof}

\subsection{Bias in RalfGANs, RaGANs, and RcGANs}

Note that we refer to the second term in RaGANs as "RaGAN2". When possible, we calculate the bias for RalfGANs, RaGAN2s, RaGANs, and RcGANs.

Let \\
$\underset{x \sim \mathbb{P}}{\mathbb{E}}[C(x)] = \mu_x$, \\
$\underset{x \sim \mathbb{P}}{Var}[C(x)] = \sigma_x^2 $, \\
$\underset{x \sim \mathbb{P}}{\mathbb{E}}[C(x)^2] = \sigma_x^2 + \mu_x^2$,\\
\newline
$\underset{y \sim \mathbb{Q}}{\mathbb{E}}[C(y)] = \mu_y$, \\
$\underset{y \sim \mathbb{Q}}{Var}[C(y)] = \sigma_y^2 $, \\
$\underset{y \sim \mathbb{Q}}{\mathbb{E}}[C(y)^2] = \sigma_y^2 + \mu_y^2$. \\

In a minibatch of size $k$, we have that $x_1, \ldots , x_k$ and $y_1, \ldots , y_k$ are iid. \\
Thus, $C(x_1), \ldots , C(x_k)$ and $C(y_1), \ldots , C(y_k)$ are also iid. \\ This means that: \\
$\mathbb{E}[C(x_i)C(x_j)] = \mathbb{E}[C(x_i)]\mathbb{E}[C(x_j)] = \mu_x^2 \hspace*{4pt}\forall i \neq j$, \\ $\mathbb{E}[C(y_i)C(y_j)] = \mathbb{E}[C(y_i)]\mathbb{E}[C(y_j)] = \mu_y^2 \hspace*{4pt}\forall i \neq j$.

\subsubsection{SGAN}

$$f(x)=\log(\sigmoid(x)) + \log(2) = -\log(1+e^{-x}) + \log(2)$$

\begin{align*}
\mathrm{Bias}^{RaSGAN}(\mathbb{P},\mathbb{Q}) &= \mathbb{E} \left[ f \left( C(x) - \frac{1}{k}\sum_{i=1}^{k} C(y_i) \right) - f \left( C(x) - \mu_y \right) \right] \\
&= \mathbb{E} \left[ -\log \left( 1 + e^{\frac{1}{k}\sum_{i=1}^{k} C(y_i) - C(x) } \right) + \log(2) + \log \left( 1 + e^{\mu_y - C(x)} \right) - \log(2) \right] \\
&= \mathbb{E} \left[ \log \left( \frac{ 1 + e^{\mu_y - C(x)} }{ 1 + e^{ \frac{1}{k}\sum_{i=1}^{k} C(y_i) - C(x) } } \right) \right] \\
&= \mathbb{E} \left[ \log \left( \frac{ e^{ C(x) } + e^{\mu_y} }{ e^{ C(x) } + e^{ \frac{1}{k}\sum_{i=1}^{k} C(y_i) } } \right) \right] \\
&= \mathbb{E} \left[ \log \left(  e^{ C(x) } + e^{\mu_y} \right) - \log \left( { e^{ C(x) } + e^{ \frac{1}{k}\sum_{i=1}^{k} C(y_i) } } \right) \right] \\
&\approx \mathbb{E} \left[ C(x) + e^{\mu_y-C(x)} - C(x) - e^{ \frac{1}{k}\sum_{i=1}^{k} C(y_i) - C(x) } \right] \\
&= \mathbb{E} \left[ \frac{e^{\mu_y}- e^{ \frac{1}{k}\sum_{i=1}^{k} C(y_i)}}{e^{C(x)}} \right]
\end{align*}

We cannot find a close form for the bias.

\subsubsection{LSGAN}

$$f(x)=-(x-1)^2 + 1$$

\begin{align*}
\mathrm{\widehat{Div}^{RaLSGAN}}(\mathbb{P},\mathbb{Q}) &=\mathbb{E} \left[ \frac{1}{k} \sum_{i=1}^{k} f \left( C(x_i) - \frac{1}{k}\sum_{j=1}^{k} C(y_j) \right) \right] \\
&= \mathbb{E} \left[ \frac{1}{k} \sum_{i=1}^{k} \left( -\left( C(x_i) - \frac{1}{k}\sum_{j=1}^{k} C(y_j) - 1 \right)^2 + 1 \right) \right] \\
&= \mathbb{E} \left[ \frac{1}{k} \sum_{i=1}^{k} \left( -C(x_i)^2 + \frac{2}{k} \sum_{j=1}^{k} C(x_i) C(y_j) + 2C(x_i) - 2\frac{1}{k} \sum_{j=1}^{k} C(y_j) - \frac{1}{k^2} \left( \sum_{j=1}^{k} C(y_j) \right)^2  \right) \right] \\
&= \frac{1}{k} \sum_{i=1}^{k} \left( -\mathbb{E} \left[C(x_i)^2 \right] + \frac{2}{k} \sum_{j=1}^{k} \mathbb{E} \left[ C(x_i) \right] \mathbb{E} \left[ C(y_j) \right] + 2 \mathbb{E} \left[C(x_i) \right] - 2\frac{1}{k} \sum_{j=1}^{k} \mathbb{E} \left[ C(y_j) \right] \right. \\
&\left. - \frac{1}{k^2}\sum_{j=1}^{k} \mathbb{E}[C(y_j)^2] - \frac{1}{k^2}\sum_{\substack{r=1 \\ r\neq j}}^{k} \sum_{\substack{j=1}}^{k} \mathbb{E} [C(y_j)] \mathbb{E}[C(y_r)] \right) \\
&= \frac{1}{k} \sum_{i=1}^{k} \left( -\sigma_x^2 -\mu_x^2 + 2 \mu_x \mu_y + 2 \mu_x - 2 \mu_y - \frac{1}{k} (\sigma_y^2 +\mu_y^2) - \frac{(k-1)}{k} \mu_y^2 \right) \\
&= -\sigma_x^2 -\mu_x^2 + 2 \mu_x \mu_y + 2 \mu_x - 2 \mu_y - \frac{1}{k} \sigma_y^2 -  \mu_y^2 \\
\end{align*}

\begin{align*}
\mathrm{\widehat{Div}^{RaLSGAN2}}(\mathbb{P},\mathbb{Q}) &=\mathbb{E} \left[ \frac{1}{k} \sum_{j=1}^{k} f \left(\frac{1}{k}\sum_{i=1}^{k}C(x_i) - C(y_j) \right) \right] \\
&= \mathbb{E} \left[ \frac{1}{k} \sum_{j=1}^{k} \left( -\left( \frac{1}{k}\sum_{i=1}^{k}C(x_i) - C(y_j) - 1 \right)^2 - 1 \right) \right] \\
&= \mathbb{E} \left[ \frac{1}{k} \sum_{j=1}^{k} \left( -C(y_j)^2 + \frac{2}{k} \sum_{x=1}^{k} C(x_i) C(y_j) - 2C(y_j) + 2\frac{1}{k} \sum_{i=1}^{k} C(x_i) - \frac{1}{k^2} \left( \sum_{i=1}^{k} C(x_i) \right)^2  \right) \right] \\
&= \frac{1}{k} \sum_{j=1}^{k} \left( -\mathbb{E} \left[C(y_j)^2 \right] + \frac{2}{k} \sum_{i=1}^{k} \mathbb{E} \left[ C(x_i) \right] \mathbb{E} \left[ C(y_j) \right] - 2 \mathbb{E} \left[C(y_j) \right] + 2\frac{1}{k} \sum_{i=1}^{k} \mathbb{E} \left[ C(x_i) \right] \right. \\
&\left. - \frac{1}{k^2}\sum_{i=1}^{k} \mathbb{E}[C(x_i)^2] - \frac{1}{k^2}\sum_{\substack{r=1 \\ r\neq i}}^{k} \sum_{\substack{i=1}}^{k} \mathbb{E} [C(x_i)] \mathbb{E}[C(x_r)] \right) \\
&= \frac{1}{k} \sum_{j=1}^{k} \left( -\sigma_y^2 -\mu_y^2 + 2 \mu_x \mu_y - 2 \mu_y + 2 \mu_x - \frac{1}{k} (\sigma_x^2 +\mu_x^2) - \frac{(k-1)}{k} \mu_x^2 \right) \\
&= -\sigma_y^2 -\mu_y^2 + 2 \mu_x \mu_y - 2 \mu_y + 2 \mu_x - \frac{1}{k} \sigma_x^2 -  \mu_x^2 \\
\end{align*}

\begin{align*}
\mathrm{Div^{RaLSGAN}}(\mathbb{P},\mathbb{Q}) &=\mathbb{E} \left[ f \left( C(x) - \mu_y \right) \right] \\
&=\mathbb{E} \left[ - \left( C(x) - \mu_y - 1 \right)^2 - 1 \right] \\
&=\mathbb{E} \left[ - C(x)^2 + 2 C(x) \mu_y + 2 C(x) - 2 \mu_y - \mu_y^2 \right] \\
&= -\sigma_x^2 - \mu_x^2 + 2\mu_x \mu_y + 2 \mu_x - 2 \mu_y - \mu_y^2
\end{align*}

\begin{align*}
\mathrm{Div^{RaLSGAN2}}(\mathbb{P},\mathbb{Q}) &=\mathbb{E} \left[ f \left(\mu_x - C(y) \right) \right] \\
&=\mathbb{E} \left[ - \left( \mu_x - C(y) - 1 \right)^2 - 1 \right] \\
&=\mathbb{E} \left[ - \mu_x^2 + 2 C(y) \mu_x - 2 C(y) + 2 \mu_x - C(y)^2 \right] \\
&= -\sigma_y^2 - \mu_y^2 + 2\mu_x \mu_y - 2 \mu_y + 2 \mu_x - \mu_x^2
\end{align*}

\begin{align*}
\mathrm{Bias^{RaLSGAN}}(\mathbb{P},\mathbb{Q}) &= \mathrm{\widehat{Div}^{RaLSGAN}}(\mathbb{P},\mathbb{Q}) - \mathrm{Div^{RaLSGAN}}(\mathbb{P},\mathbb{Q}) \\
&= -\sigma_x^2 -\mu_x^2 + 2 \mu_x \mu_y + 2 \mu_x - 2 \mu_y - \frac{1}{k} \sigma_y^2 -  \mu_y^2 +\sigma_x^2 + \mu_x^2 - 2\mu_x \mu_y - 2 \mu_x + 2 \mu_y + \mu_y^2 \\
&=- \frac{1}{k} \sigma_y^2 \\
\end{align*}

\begin{align*}
\mathrm{Bias^{RaLSGAN2}}(\mathbb{P},\mathbb{Q}) &= \mathrm{\widehat{Div}^{RaLSGAN2}}(\mathbb{P},\mathbb{Q}) - \mathrm{Div^{RaLSGAN2}}(\mathbb{P},\mathbb{Q}) \\
&= -\sigma_y^2 -\mu_y^2 + 2 \mu_x \mu_y - 2 \mu_y + 2 \mu_x - \frac{1}{k} \sigma_x^2 -  \mu_x^2 +\sigma_y^2 + \mu_y^2 - 2\mu_x \mu_y + 2 \mu_y - 2 \mu_x + \mu_x^2 \\
&=- \frac{1}{k} \sigma_x^2 \\
\end{align*}

\begin{align*}
\mathrm{Bias}^{RalfLSGAN} &= \mathrm{Bias}^{RaLSGAN}(\mathbb{P},\mathbb{Q}) + \mathrm{Bias}^{RaLSGAN2}(\mathbb{Q},\mathbb{P}) \\
&= -\frac{1}{k} \sigma_y^2 -\frac{1}{k} \sigma_x^2 \\
&= -\frac{1}{k} \left(\sigma_x^2 + \sigma_y^2 \right)
\end{align*}

Let \\
$\hat{\sigma}_x^2 = \frac{1}{(k-1)} \sum_{i=1}^{k} \left( C(x_i) - \frac{1}{k} \sum_{i=1}^{k} C(x_j) \right)$, \\  $\hat{\sigma}_y^2 = \frac{1}{(k-1)} \sum_{i=1}^{k} \left( C(y_i) - \frac{1}{k} \sum_{i=1}^{k} C(y_j) \right)$.

We know that $\hat{\sigma}_x^2$ and $\hat{\sigma}_y^2$ are unbiased estimators of $\sigma_x^2$ and $\sigma_y^2$ respectively. \\
Thus, if we add $\frac{1}{k} \hat{\sigma}_y^2$ to the objective function of RalfLSGAN and $\frac{1}{k} ( \hat{\sigma}_x^2 + \hat{\sigma}_y^2)$ to the objective function of RaLSGAN, we have that the new objective functions are unbiased.

\begin{align*}
\mathrm{\widehat{Div}^{RcLSGAN}}(\mathbb{P},\mathbb{Q}) &=\mathbb{E} \left[ \frac{1}{k} \sum_{i=1}^{k} f \left( C(x_i) - \frac{1}{2k}\sum_{j=1}^{k} \left( C(x_j) + C(y_j) \right)  \right) \right] \\
&= \mathbb{E} \left[ \frac{1}{k} \sum_{i=1}^{k} \left( -\left( C(x_i) - \frac{1}{2k}\sum_{j=1}^{k} \left( C(x_j) + C(y_j) \right) - 1 \right)^2 + 1 \right) \right] \\
&= \mathbb{E} \left[ \frac{1}{k} \sum_{i=1}^{k} \left( -C(x_i)^2 + \frac{1}{k} \sum_{j=1}^{k} C(x_i) \left(C(x_j) + C(y_j) \right) + 2C(x_i) - \frac{1}{k} \sum_{j=1}^{k} C(x_j) - \frac{1}{k} \sum_{j=1}^{k} C(y_j) \right.\right. \\
& \left.\left. - \frac{1}{4k^2} \left( \sum_{j=1}^{k} C(x_j) + C(y_j) \right)^2  \right) \right] \\
&= \frac{1}{k} \sum_{i=1}^{k} \left( -\mathbb{E} \left[C(x_i)^2 \right] + \frac{1}{k} \mathbb{E} \left[ C(x_i)^2 \right] + \frac{1}{k} \sum_{\substack{j=1 \\ j\neq i}}^{k} \mathbb{E} \left[ C(x_i) \right] \mathbb{E} \left[ C(x_j) \right] + \frac{1}{k} \sum_{j=1}^{k} \mathbb{E} \left[ C(x_i) \right] \mathbb{E} \left[ C(y_j) \right] \right. \\
&\left. + 2 \mathbb{E} \left[C(x_i) \right] - \frac{1}{k} \sum_{j=1}^{k} \mathbb{E} \left[ C(x_j) \right] - \frac{1}{k} \sum_{j=1}^{k} \mathbb{E} \left[ C(y_j) \right] - \frac{1}{4k^2}\sum_{j=1}^{k} \mathbb{E}[(C(x_j)+C(y_j))^2] \right. \\
&\left. - \frac{1}{4k^2}\sum_{\substack{r=1 \\ r\neq j}}^{k} \sum_{\substack{j=1}}^{k} \mathbb{E} [C(x_i)+C(y_i)] \mathbb{E}[C(x_r)+C(y_r)] \right) \\
&= \left(\frac{1}{k}-1\right)\left(\sigma_x^2 + \mu_x^2\right) + \frac{(k-1)}{k}\mu_x^2 + \mu_x \mu_y + 2 \mu_x - \mu_x - \mu_y \\
& - \frac{1}{4k} ((\sigma_x^2 +\mu_x^2) + 2\mu_x\mu_y  + (\sigma_y^2 +\mu_y^2)) - \frac{(k-1)}{4k} (\mu_x^2 + 2\mu_x\mu_y + \mu_y^2) \\
&= \frac{(1-k)}{k}\sigma_x^2 + \mu_x \mu_y + \mu_x - \mu_y - \frac{1}{4} \mu_x^2 - \frac{1}{2} \mu_x\mu_y - \frac{1}{4} \mu_y^2 - \frac{1}{4k}\sigma_x^2 - \frac{1}{4k}\sigma_y^2 \\
&= \frac{(.75-k)}{k}\sigma_x^2 - \frac{1}{4k}\sigma_y^2 - \frac{1}{4} \mu_x^2 - \frac{1}{4} \mu_y^2 + \frac{1}{2} \mu_x\mu_y + \mu_x - \mu_y
\end{align*}

\begin{align*}
\mathrm{\widehat{Div}^{RcLSGAN}}(\mathbb{P},\mathbb{Q}) &= \mathbb{E} \left[ \frac{1}{k} \sum_{i=1}^{k} \left( -\left( C(y_i) - \frac{1}{2k}\sum_{j=1}^{k} \left( C(x_j) + C(y_j) \right) + 1 \right)^2 + 1 \right) \right] \\
&= \mathbb{E} \left[ \frac{1}{k} \sum_{i=1}^{k} \left( -C(y_i)^2 + \frac{1}{k} \sum_{j=1}^{k} C(y_i) \left(C(x_j) + C(y_j) \right) - 2C(y_i) + \frac{1}{k} \sum_{j=1}^{k} C(x_j) + \frac{1}{k} \sum_{j=1}^{k} C(y_j) \right.\right. \\
& \left.\left. - \frac{1}{4k^2} \left( \sum_{j=1}^{k} C(x_j) + C(y_j) \right)^2  \right) \right] \\
&= \frac{1}{k} \sum_{i=1}^{k} \left( -\mathbb{E} \left[C(y_i)^2 \right] + \frac{1}{k} \mathbb{E} \left[ C(y_i)^2 \right] + \frac{1}{k} \sum_{\substack{j=1 \\ j\neq i}}^{k} \mathbb{E} \left[ C(y_i) \right] \mathbb{E} \left[ C(y_j) \right] + \frac{1}{k} \sum_{j=1}^{k} \mathbb{E} \left[ C(x_i) \right] \mathbb{E} \left[ C(y_j) \right] \right. \\
&\left. - 2 \mathbb{E} \left[C(y_i) \right] + \frac{1}{k} \sum_{j=1}^{k} \mathbb{E} \left[ C(x_j) \right] + \frac{1}{k} \sum_{j=1}^{k} \mathbb{E} \left[ C(y_j) \right] - \frac{1}{4k^2}\sum_{j=1}^{k} \mathbb{E}[(C(x_j)+C(y_j))^2] \right. \\
&\left. - \frac{1}{4k^2}\sum_{\substack{r=1 \\ r\neq j}}^{k} \sum_{\substack{j=1}}^{k} \mathbb{E} [C(x_i)+C(y_i)] \mathbb{E}[C(x_r)+C(y_r)] \right) \\
&= \left(\frac{1}{k}-1\right)\left(\sigma_y^2 + \mu_y^2\right) + \frac{(k-1)}{k}\mu_y^2 + \mu_x \mu_y - 2\mu_y + \mu_x + \mu_y \\
& - \frac{1}{4k} ((\sigma_x^2 +\mu_x^2) + 2\mu_x\mu_y  + (\sigma_y^2 +\mu_y^2)) - \frac{(k-1)}{4k} (\mu_x^2 + 2\mu_x\mu_y + \mu_y^2) \\
&= \frac{(1-k)}{k}\sigma_y^2 + \mu_x \mu_y + \mu_x - \mu_y - \frac{1}{4} \mu_x^2 - \frac{1}{2} \mu_x\mu_y - \frac{1}{4} \mu_y^2 - \frac{1}{4k}\sigma_x^2 - \frac{1}{4k}\sigma_y^2 \\
&= \frac{(.75-k)}{k}\sigma_y^2 - \frac{1}{4k}\sigma_x^2 - \frac{1}{4} \mu_x^2 - \frac{1}{4} \mu_y^2 + \frac{1}{2} \mu_x\mu_y + \mu_x - \mu_y
\end{align*}

\begin{align*}
\mathrm{Div^{RcLSGAN}}(\mathbb{P},\mathbb{Q}) &=\mathbb{E} \left[ f \left( C(x) - \frac{(\mu_x+\mu_y)}{2} \right) \right] \\
&=\mathbb{E} \left[ - \left( C(x) - \frac{(\mu_x+\mu_y)}{2} - 1 \right)^2 - 1 \right] \\
&=\mathbb{E} \left[ - C(x)^2 + C(x) (\mu_x+\mu_y) + 2 C(x) - (\mu_x+\mu_y) - \frac{(\mu_x+\mu_y)^2}{4} \right] \\
&= -\sigma_x^2 - \mu_x^2 + \mu_x^2 + \mu_x\mu_y + \mu_x - \mu_y - \frac{1}{4}(\mu_x^2 + \mu_y^2 + 2\mu_x\mu_y) \\
&= -\sigma_x^2 + \frac{1}{2}\mu_x\mu_y + \mu_x - \mu_y - \frac{1}{4}\mu_x^2 - \frac{1}{4}\mu_y^2
\end{align*}

\begin{align*}
\mathrm{Div^{RcLSGAN2}}(\mathbb{P},\mathbb{Q}) &=\mathbb{E} \left[ f \left( C(x) - \mu_y \right) \right] \\
&=\mathbb{E} \left[ - \left( C(y) - \frac{(\mu_x+\mu_y)}{2} + 1 \right)^2 - 1 \right] \\
&=\mathbb{E} \left[ - C(y)^2 + C(y) (\mu_x+\mu_y) - 2 C(y) + (\mu_x+\mu_y) - \frac{(\mu_x+\mu_y)^2}{4} \right] \\
&= -\sigma_y^2 - \mu_y^2 + \mu_y^2 + \mu_x\mu_y + \mu_x - \mu_y - \frac{1}{4}(\mu_x^2 + \mu_y^2 + 2\mu_x\mu_y) \\
&= -\sigma_x^2 + \frac{1}{2}\mu_x\mu_y + \mu_x - \mu_y - \frac{1}{4}\mu_x^2 - \frac{1}{4}\mu_y^2
\end{align*}

\begin{align*}
\mathrm{Bias^{RaLSGAN}}(\mathbb{P},\mathbb{Q}) &= \mathrm{\widehat{Div}^{RaLSGAN}}(\mathbb{P},\mathbb{Q}) - \mathrm{Div^{RaLSGAN}}(\mathbb{P},\mathbb{Q}) \\
&= \frac{3}{4k}\sigma_x^2 - \frac{1}{4k}\sigma_y^2
\end{align*}

\begin{align*}
\mathrm{Bias^{RaLSGAN2}}(\mathbb{P},\mathbb{Q}) &= \mathrm{\widehat{Div}^{RaLSGAN2}}(\mathbb{P},\mathbb{Q}) - \mathrm{Div^{RaLSGAN2}}(\mathbb{P},\mathbb{Q}) \\
&= \frac{3}{4k}\sigma_y^2 - \frac{1}{4k}\sigma_x^2
\end{align*}

\begin{align*}
\mathrm{Bias}^{RalfLSGAN} &= \mathrm{Bias}^{RaLSGAN}(\mathbb{P},\mathbb{Q}) + \mathrm{Bias}^{RaLSGAN2}(\mathbb{Q},\mathbb{P}) \\
&= \frac{3}{4k}\sigma_x^2 - \frac{1}{4k}\sigma_y^2 + \frac{3}{4k}\sigma_y^2 - \frac{1}{4k}\sigma_x^2 \\
&= \frac{1}{2k} \left(\sigma_x^2 + \sigma_y^2\right)
\end{align*}

Let \\
$\hat{\sigma}_x^2 = \frac{1}{(k-1)} \sum_{i=1}^{k} \left( C(x_i) - \frac{1}{k} \sum_{i=1}^{k} C(x_j) \right)$, \\  $\hat{\sigma}_y^2 = \frac{1}{(k-1)} \sum_{i=1}^{k} \left( C(y_i) - \frac{1}{k} \sum_{i=1}^{k} C(y_j) \right)$.

We know that $\hat{\sigma}_x^2$ and $\hat{\sigma}_y^2$ are unbiased estimators of $\sigma_x^2$ and $\sigma_y^2$ respectively. \\
Thus, if we subtract $\frac{1}{2k} ( \hat{\sigma}_x^2 + \hat{\sigma}_y^2)$ to the objective function of RcLSGAN, we have that the new objective functions are unbiased.

\subsubsection{HingeGAN}

$$f(x)=-\max(0,1-x) + 1$$
For simplicity:\\
Let $x'=C(x)$, $y_i' = C(y_i)$, $p(x)$ and $q(x)$ be the probability density functions of $x'$ and $y'_i$.

\begin{align*}
\mathrm{Div^{RaHingeGAN}}(\mathbb{P},\mathbb{Q}) &=\mathbb{E} \left[ f \left( C(x) - \frac{1}{k}\sum_{i=1}^{k} C(y_i) \right) \right] \\
&= \mathbb{E} \left[ -\max\left(0,  1 + \frac{1}{k}\sum_{i=1}^{k} y_i' - x' \right)+1 \right] \\
&= -\int_{-\infty}^{\infty}... \int_{-\infty}^{\infty}\int_{-\infty}^{1+\frac{1}{k}\sum_{i=1}^{k}y_i'} \left( 1+\frac{1}{k}\sum_{i=1}^{k} y_i' - x\right)p(x)q(y)...q(y)dx dy_1 ... dy_k
\end{align*}
This is non-linear and we cannot derive a close-form.

\subsection{Architecture}

\begin{tabular}{c}
	Generator \\
	\toprule\midrule
	$z \in \mathbb{R}^{128} \sim N(0,I)$ \\
	\midrule
	linear, 128 -> 512*4*4 \\
	\midrule
	Reshape, 512*4*4 -> 512 x 4 x 4 \\
	\midrule
	ConvTranspose2d 4x4, stride 2, pad 1, 512->256 \\
	\midrule
	BN and ReLU \\
	\midrule
	ConvTranspose2d 4x4, stride 2, pad 1, 256->128 \\
	\midrule
	BN and ReLU \\
	\midrule
	ConvTranspose2d 4x4, stride 2, pad 1, 128->64 \\
	\midrule
	BN and ReLU \\
	\midrule
	ConvTranspose2d 3x3, stride 1, pad 1, 64->3 \\
	\midrule
	Tanh \\
	\bottomrule
\end{tabular} 
\quad
\begin{tabular}{c}
	Discriminator \\
	\toprule\midrule
	$x \in \mathbb{R}^{\text{3x32x32}}$ \\
	\midrule
	Conv2d 3x3, stride 1, pad 1, 3->64 \\
	\midrule
	LeakyReLU 0.1 \\
	\midrule
	Conv2d 4x4, stride 2, pad 1, 64->64 \\
	\midrule
	LeakyReLU 0.1 \\
	\midrule
	Conv2d 3x3, stride 1, pad 1, 64->128 \\
	\midrule
	LeakyReLU 0.1 \\
	\midrule
	Conv2d 4x4, stride 2, pad 1, 128->128 \\
	\midrule
	LeakyReLU 0.1 \\
	\midrule
	Conv2d 3x3, stride 1, pad 1, 128->256 \\
	\midrule
	LeakyReLU 0.1 \\
	\midrule
	Conv2d 4x4, stride 2, pad 1, 256->256 \\
	\midrule
	LeakyReLU 0.1 \\
	\midrule
	Conv2d 3x3, stride 1, pad 1, 256->512 \\
	\midrule
	Reshape, 512 x 4 x 4 -> 512*4*4 \\
	\midrule
	linear, 512*4*4 -> 1 \\
	\bottomrule
\end{tabular}

\end{document}